\documentclass{article}

\usepackage{enumitem}

\usepackage[preprint]{neurips_2026}

\usepackage[utf8]{inputenc} % allow utf-8 input
\usepackage[T1]{fontenc}    % use 8-bit T1 fonts
\usepackage{hyperref}       % hyperlinks
\usepackage{url}            % simple URL typesetting
\usepackage{booktabs}       % professional-quality tables
\usepackage{amsfonts}       % blackboard math symbols
\usepackage{nicefrac}       % compact symbols for 1/2, etc.
\usepackage{microtype}      % microtypography
\usepackage{xcolor}         % colors
\usepackage{subcaption}
\usepackage{amsmath}
\usepackage{amssymb}
\usepackage{mathtools}
\usepackage{amsthm}
\usepackage{amsfonts} 
\usepackage{enumitem}
\usepackage{wrapfig}
\usepackage{multirow}
\usepackage{algorithm}
\usepackage{algorithmic}
\usepackage{float}
\usepackage{mdframed}
\usepackage{graphicx}
\usepackage{algorithm}
\usepackage{caption}
\DeclareMathOperator{\Var}{Var}
\usepackage{listings}
\definecolor{lstbg}{gray}{0.97}
\definecolor{lstkw}{RGB}{0,0,180}
\definecolor{lstcomment}{RGB}{0,120,0}
\definecolor{lststring}{RGB}{163,21,21}
\lstdefinestyle{pythonstyle}{
  language=Python,
  basicstyle=\ttfamily\footnotesize,
  keywordstyle=\color{lstkw}\bfseries,
  commentstyle=\color{lstcomment}\itshape,
  stringstyle=\color{lststring},
  numbers=left,
  numberstyle=\tiny\color{gray},
  numbersep=6pt,
  backgroundcolor=\color{lstbg},
  frame=single,
  framesep=2mm,
  breaklines=true,
  breakatwhitespace=true,
  showstringspaces=false,
  tabsize=4,
  captionpos=b,
  columns=fullflexible,
  keepspaces=true,
}
\newcommand{\mat}[1]{\mathbf{#1}}
\definecolor{darkgreen}{RGB}{0,120,0}
\usepackage[capitalize,noabbrev]{cleveref}

\theoremstyle{plain}
\newtheorem{theorem}{Theorem}[section]

\newtheorem{lemma}[theorem]{Lemma}
\newtheorem{corollary}[theorem]{Corollary}
\theoremstyle{definition}

\newtheorem{assumption}[theorem]{Assumption}

\usepackage{amsmath}

\usepackage[textsize=tiny]{todonotes}
\usepackage[most]{tcolorbox}

\newtcolorbox{remarkbox}{
  enhanced,
  breakable,
  colback=blue!2,
  colframe=black!15,
  boxrule=0.3pt,
  arc=1pt,
  left=8pt,
  right=8pt,
  top=6pt,
  bottom=6pt
}
\title{\textsc{Muon+}: Towards More Effective Muon via One Additional Normalization Step for LLM Pre-training}

\author{%
  Ruijie Zhang,\quad Yequan Zhao, \quad Ziyue Liu,\quad Zhengyang Wang,\\
  \textbf{Yupeng Su,\quad Liyan Tan,\quad Zheng Zhang\textsuperscript{$\dagger$}} \\
  University of California at Santa Barbara\\
  {ruijiezhang@ucsb.edu}, \\ \quad zhengzhang@ece.ucsb.edu
}

\begin{document}

\maketitle

\begin{abstract}
Muon has recently emerged as a strong optimizer for large language model pre-training, orthogonalizing the momentum matrix via Newton--Schulz polar iterations. A natural intuition is that polar iterations, by flattening the singular spectrum to all ones, should also eliminate column- and row-wise norm imbalance in the update. We show that this is not true in practice: practical polar steps can substantially \emph{amplify} the imbalance. We term this the \emph{post-polar imbalanced update} problem, and prove that such imbalance tightens the second-order term in a blockwise descent analysis, weakening Muon's per-step descent guarantee. Motivated by this analysis, we propose \textsc{Muon+}, a one-line fix that inserts a single normalization step after polar orthogonalization. \textsc{Muon+} adds no optimizer state. Across pre-training experiments on GPT and LLaMA models from 60M to 7B parameters, spanning both compute-optimal budgets and extended token-to-parameter ratios up to approximately 200, \textsc{Muon+} consistently outperforms Muon in terms of training and validation perplexity, leading to significant overall pre-training speedup.
\end{abstract}

\section{Introduction}

Based on the empirical observation of scaling laws \cite{kaplan2020scaling, hoffmann2022training, kumar2025scaling}, powerful foundation models such as GPT, DeepSeek, LLaMA, and Gemini \cite{achiam2023gpt,liu2024deepseek, grattafiori2024llama, team2023gemini} have been trained and widely deployed. Nevertheless, as the sizes of both model parameters and training datasets reach extreme levels, the computational cost of pre-training has become prohibitively high. This challenge has motivated increasing research dedicated to improving pre-training efficiency \cite{mehmood2023efficient, han2024sltrain, cesistasqueezing, zhao2024galore, zhang2025lax,liu2025cola}, with a particular emphasis on the critical role of optimizers. Although Adam \cite{kingma2014adam} and AdamW \cite{loshchilov2017decoupled} are still the dominant optimizers, numerous efficient optimizers have been proposed to reduce the computing or memory cost of large-scale pre-training \cite{kingma2014adam, loshchilov2017decoupled,liu2024sophia,jordan2024muon,yuan2024mars,vyas2025soap,Li_2018,li2018preconditionermatrixliegroup,pooladzandi2024curvatureinformedsgdgeneralpurpose,li2022blackboxliegroup,li2024stochastichessianfittingslie,pethick2025trainingdeeplearningmodels}.

Among these, the Muon optimizer~\cite{jordan2024muon} has recently emerged as a strong choice for pre-training. Its key idea is to orthogonalize the momentum matrix via Newton--Schulz iterations, flattening its singular spectrum to mitigate ``rank collapse'' of the update. Muon has been shown to scale to massive foundation models~\cite{liu2025muon} and is now integral to production-scale systems such as Kimi and GLM~\cite{team2025kimi,ding2025kimi,zeng2025glm,team2025kimivl}. A rapidly growing literature further studies Muon's efficiency, scalability, and theoretical properties~\cite{zhang2026teon,bernstein2025deriving,khaled2025muonbp,amsel2025polar,li2025normuon,kovalev2025understanding}.

\paragraph{Muon update rule.}
Unlike Adam/SGD-based optimizers, Muon \cite{jordan2024muon} operates on a matrix rather than a vector. By enforcing orthogonalization on the gradient, Muon prevents the rank collapse of the gradient by replacing the singular value matrix with an identity matrix.
Let $\eta$ and $\mu$ denote the learning rate and the momentum coefficient, respectively. Assume that $\mat{W}_t \in \mathbb{R}^{m \times n}$ is the layer being adapted at iteration $t$, $\mat{G}_t \in \mathbb{R}^{m \times n}$ is its stochastic gradient, and $\mat{M}_t$ is the gradient momentum at iteration $t$. The Muon update is given by
\begin{equation}
\begin{aligned}
\mat{M}_t &= \mu \mat{M}_{t-1} + (1 - \mu) \mat{G}_t \\
\mat{O}_t &= \mathrm{Ortho}(\mat{M}_t) \\
\mat{W}_{t} &= \mat{W}_{t-1} - \eta \cdot \sqrt{m/n} \cdot \mat{O}_t
\end{aligned}
\label{eq:muon_update_rule}
\end{equation}
where $\mathrm{Ortho}(\cdot)$ denotes the semi-orthogonal matrix function closest to the input matrix \cite{higham2008functions}. Specifically, if the SVD of the input matrix $\mat{M}$ is $\mat{M} = \mat{U} \mat{\Sigma} \mat{V}^T$, then $\mathrm{Ortho}(\mat{M}) := \mat{U} \mat{V}^T$. In practice, the Newton-Schulz iteration process \cite{higham2008functions} is commonly used to approximate the SVD. The dimensional pre-factor $\sqrt{m/n}$ was suggested by \cite{bernstein2025deriving} for better scalability.

In this work, we identify a limitation of Muon, which we term the \emph{post-polar imbalanced update} problem, and propose a simple yet effective extension, \textsc{Muon+}, to address it. We summarize our contributions as follows:
\begin{itemize}[leftmargin=*]
    \item \textbf{A counter-intuitive phenomenon in Muon: update imbalance.} We identify the \emph{post-polar imbalanced update} problem: under practical Newton--Schulz settings, polar iterations can \emph{amplify} rather than reduce column/row norm imbalance.

    \item \textbf{Theoretical analysis.} We prove (i) that update imbalance weakens Muon's convergence (Theorem~\ref{thm:imbalance_bad_informal}), and (ii) that the expected column/row variance shift of one polar step admits a closed-form characterization with a provably positive amplification region (Lemma~\ref{lemma:var_diff} and Corollary~\ref{cor:positive_shift_exists}).

    \item \textbf{A simple yet effective fix: \textsc{Muon+}.} A single post-orthogonalization normalization mitigates the imbalance with zero memory overhead and negligible computing overhead.

    \item \textbf{Empirical validation on GPT and LLaMA.} On GPT and LLaMA pre-training from 60M to 7B parameters, spanning compute-optimal budgets and token-to-parameter ratios up to $\sim$200, \textsc{Muon+} consistently outperforms Muon across all evaluated settings. It reduces validation perplexity by up to \textbf{2.02}, improves average downstream zero-shot accuracy by up to \textbf{1.3\%} across \textbf{7} tasks, and speeds up the pre-training up to \textbf{37.1\%}, while requiring \textbf{zero} additional optimizer states.
\end{itemize}
\begin{figure}[t]
    \centering
    \includegraphics[width=\linewidth]{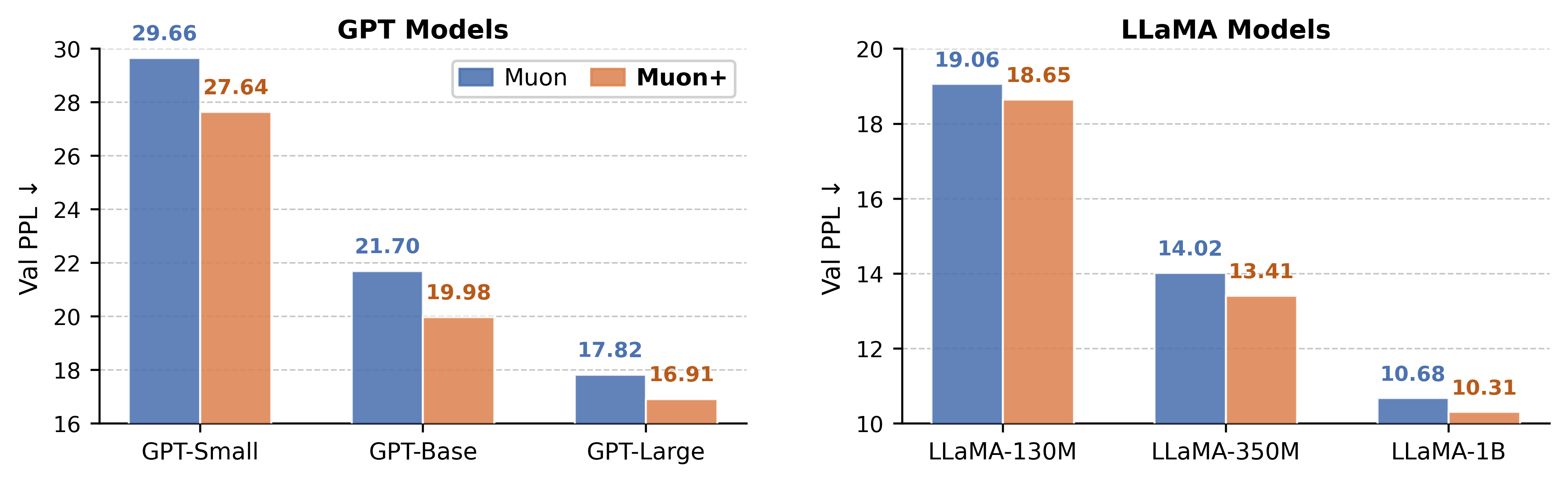}
    \caption{Pre-training GPT and LLaMA models at scales ranging from 130M to 1B parameters under compute-optimal settings. Quantitative results are provided in Section~\ref{sec:experiments}. \textbf{\textsc{Muon+}} consistently outperforms \textsc{Muon} across all runs. We also conduct larger-scale and overtraining experiments for both GPT and LLaMA; the results are presented in Section~\ref{sec:experiments}.}
    \label{fig:muonp_vs_muon}
\end{figure}
\section{Method}
\label{sec:method}
This section first shows that Muon's polar iteration can \emph{amplify} column/row update imbalance (Section~\ref{sec:amp}), then proves that such imbalance weakens Muon's local descent guarantees (Section~\ref{sec:imbalanced_undesired}), and finally presents \textsc{Muon+} as a fix (Section~\ref{sec:muon+_method}). Proofs are deferred to Appendices~\ref{apx:imbalance_convergence} and~\ref{apx:amp_proof}.
\begin{figure}[t]
    \centering

    \begin{subfigure}[t]{\linewidth}
        \centering
        \includegraphics[width=\linewidth]{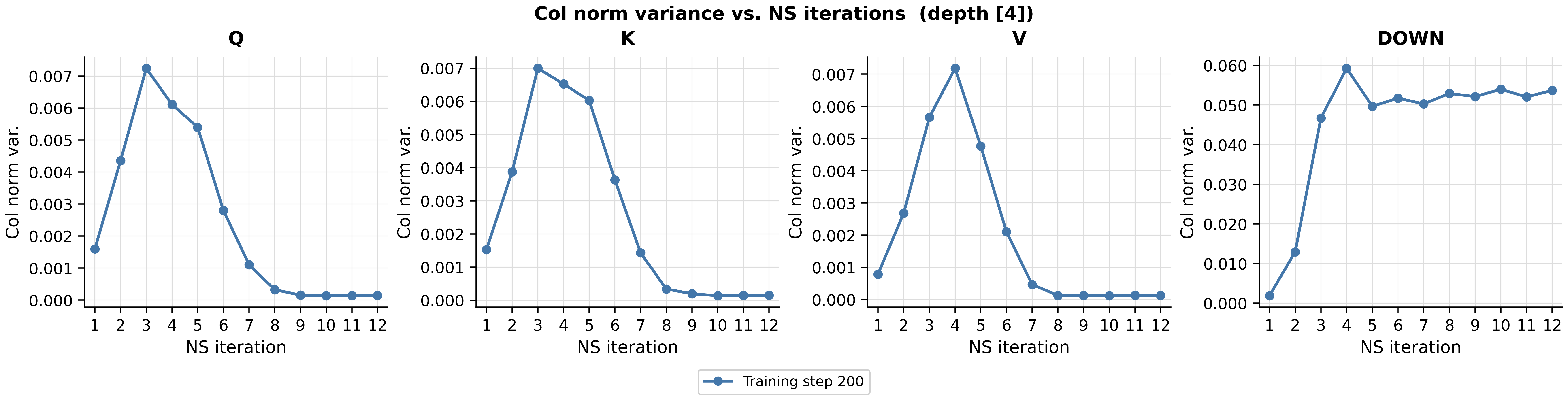}
        \caption{Variance trend with respect to the number of Newton--Schulz iterations. The imbalance remains amplified even under a typical number of Newton--Schulz iterations (e.g., 5). For non-square layers, the variance in one direction cannot be reduced to zero.}
        \label{fig:var_trend_ns}
    \end{subfigure}

    \vspace{0.8em}

    \begin{subfigure}[t]{\linewidth}
        \centering
        \includegraphics[width=\linewidth]{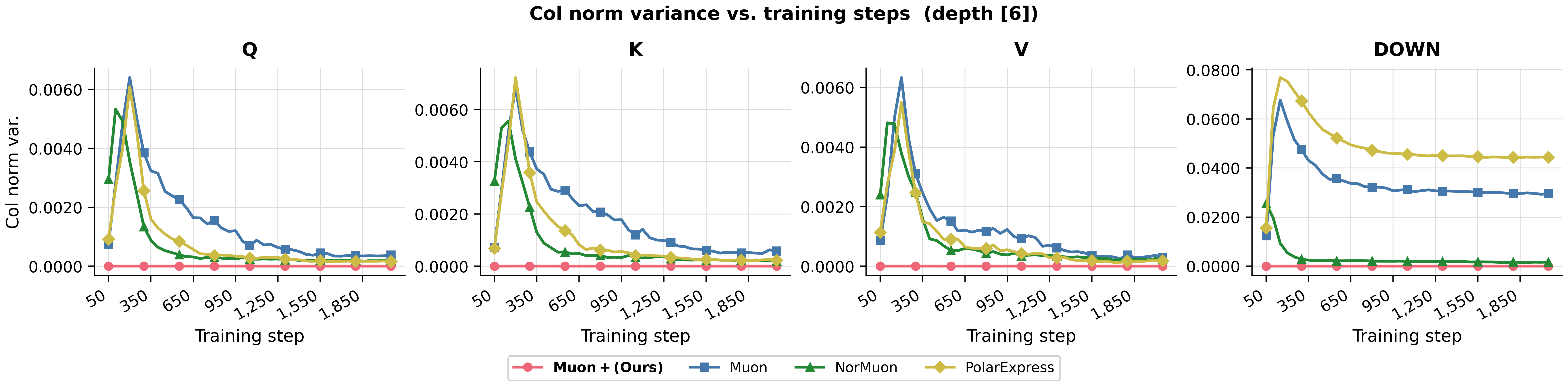}
        \caption{Variance trend under different Muon variants. Existing variants fail to resolve the issue, whereas Muon+ maintains zero variance throughout training (see Figure~\ref{fig:var_trend_all} for more details).}
        \label{fig:muon_variants_var}
    \end{subfigure}

    \caption{Comparison of variance behavior under different settings during the realistic training of a 60M LLaMA model on 1.1B FineWeb tokens.}
    \label{fig:variance_comparison}
\end{figure}

\subsection{Update Imbalance Caused by Muon's Polar Iterations }

\label{sec:amp}

\paragraph{Update imbalance.}
Muon is designed to reduce \emph{spectral} imbalance by flattening the singular values of the update matrix. In this work, we study a complementary notion of imbalance in the parameter space. We define \emph{update imbalance} as highly non-uniform magnitudes across the elements of the update matrix (e.g., some columns or rows having much larger $\ell_2$ norm than others (Fig.~\ref{fig:3d_imbalance})). We measure this imbalance by the variance of column or row norms. 
For any matrix $\mat{X}$, define
\[
s(\mat{X})
:=
\bigl(\|\mat{X}_{1:}\|_2^2,\dots,\|\mat{X}_{n:}\|_2^2\bigr).
\]
The row-norm imbalance is measured by
\[
\Var\!\bigl(s(\mat{X})\bigr)
=
\frac{1}{n}\sum_{i=1}^n
\left(
\|\mat{X}_{i:}\|_2^2-\frac{1}{n}\|\mat{X}\|_F^2
\right)^2.
\]
The column-norm imbalance can be defined likewise.

\paragraph{Post-polar imbalanced update.}
\begin{wrapfigure}{r}{0.52\textwidth}
    \centering
    \vspace{-0.5\baselineskip}
    \includegraphics[width=0.50\textwidth]{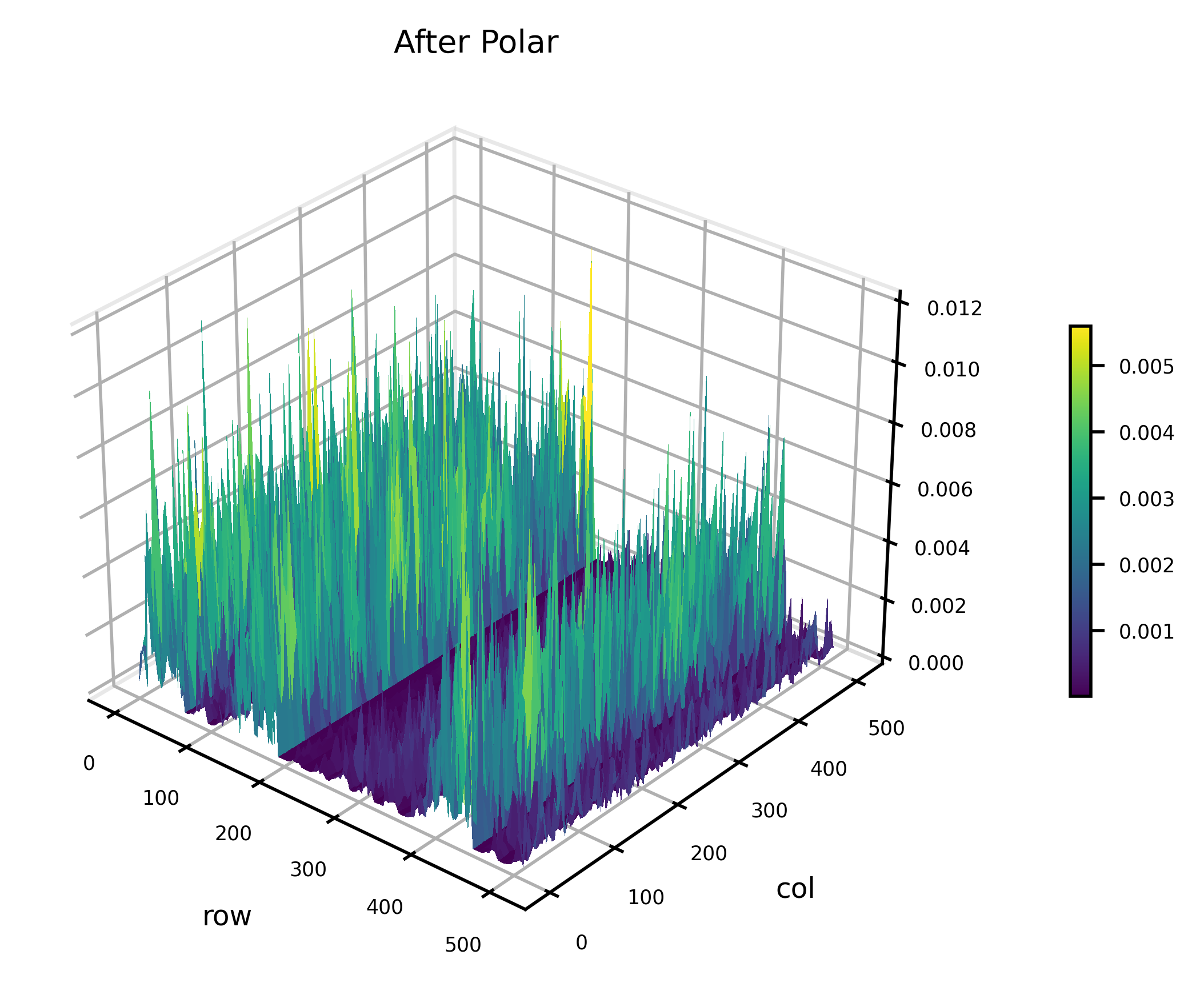}
    \caption{Layer-6 Q Projection Update matrix of Muon. The update matrix is highly unbalanced. The trends over iteration/training are shown in Figure~\ref{fig:variance_comparison}.}
    \label{fig:3d_imbalance}
    \vspace{-0.5\baselineskip}
\end{wrapfigure}
For a square layer, Muon replaces $\mat M = \mat U\mat\Sigma\mat V^\top$ with $\mat Q=\mat U\mat V^\top$, \textbf{whose rows and columns all have unit $\ell_2$ norm. Ideally, one would expect that the polar step should eliminate row/column norm imbalance in the update \(\mat Q\)}. In practice we find the opposite: with a typical number of Newton-Schulz iterations (5 steps), the variance is significantly amplified (Figure~\ref{fig:variance_comparison}). In principle, one may mitigate the imbalance by running a large number of Newton-Schulz iterations, but this is computationally infeasible in LLM pre-training. For a non-squared matrix, we cannot reduce the variance of one direction to $0$ even if we use exact SVD. 

To understand why this phenomenon occurs, we characterize the variance shift induced by one polar step. Let
\[
\mat{M} \in \mathbb{R}^{n \times n}, \quad \|\mat{M}\|_F^2 = 1, \quad  \mat{M} = \mat{U \Sigma V}^\top
\]
where
\[
\Sigma = \operatorname{diag}(\sigma_1,\dots,\sigma_n), \qquad 1 \ge \sigma_k \ge 0,
\]
Define one Newton-Schulz step by
\[
\mat{Q}
:= a\mat{M}  + b(\mat{M} \mat{M} ^\top)\mat{M}  + c(\mat{M} \mat{M} ^\top)^2\mat{M} ,
\]
with coefficients \(a,b,c \in \mathbb{R}\). Introduce the quintic polynomial
\[
\varphi(x) := ax + bx^3 + cx^5.
\]
Then
\[
\mat{Q} = \mat{U}  \varphi(\Sigma) \mat{V} ^\top.
\]

\begin{lemma}[Expected variance shift of one polar step]
\label{lemma:var_diff}
Assume that \(\mat{U}\) is Haar distributed on the orthogonal group\footnote{The orthogonal group is the group of all real $n \times n$ matrices \(\mat{U}\) satisfying \(\mat{U}^\top\mat{U}=\mat{I}\).} Let \(\mat{M}\) and \(\mat{Q}\) be defined as above, and let
\[
\lambda_k:=\sigma_k^2,
\qquad
g(\lambda):=\varphi(\sqrt{\lambda})^2
=\lambda(a+b\lambda+c\lambda^2)^2.
\]
Then, for each pair \(1\le i<j\le n\), there exists
\[
\xi_{ij}\in\bigl(\min\{\lambda_i,\lambda_j\},\,\max\{\lambda_i,\lambda_j\}\bigr)
\]
such that
\begin{equation}
\label{eq:var_shift_mvt}
\mathbb{E}\!\left[\Var(s(\mat{Q}))-\Var(s(\mat{M}))\right]
=
\frac{2}{n^2(n+2)}
\sum_{1\le i<j\le n}
(\lambda_i-\lambda_j)^2\bigl(g'(\xi_{ij})^2-1\bigr).
\end{equation}
\end{lemma}

\begin{corollary}[Existence of a positive expected variance shift]
\label{cor:positive_shift_exists}
Let \(\varphi(x)\) and \(g(\lambda)\) be defined as above.
Assume that \(a=\varphi'(0)>1\). Then there exists \(\delta>0\) such that, whenever
\[
\lambda_1,\dots,\lambda_n\in[0,\delta]
\]
are not all equal,
\[
\mathbb{E}\!\left[\Var(s(\mat Q))-\Var(s(\mat M))\right]>0.
\]
In particular, the expected variance shift of one polar step can be positive.
\end{corollary}
Detailed proof is provided in Appendix~\ref{apx:amp_proof}
%\begin{remarkbox}

\paragraph{Remark.}
Lemma~\ref{lemma:var_diff} and Corollary~\ref{cor:positive_shift_exists} suggest a natural partition of the spectral interval into amplification and reduction regions. Specifically, the derivative profile of \(g\), which is fully determined by the Newton--Schulz coefficients \(a,b,c\), divides the spectrum into intervals where \(|g'(\lambda)|>1\) and intervals where \(|g'(\lambda)|<1\). This partition determines the sign of the pairwise contributions to the variance shift: pairs associated with the former regions contribute positively, whereas pairs associated with the latter regions contribute negatively.
%\end{remarkbox}
\begin{wrapfigure}{r}{0.42\linewidth}
    \centering
    \includegraphics[width=\linewidth]{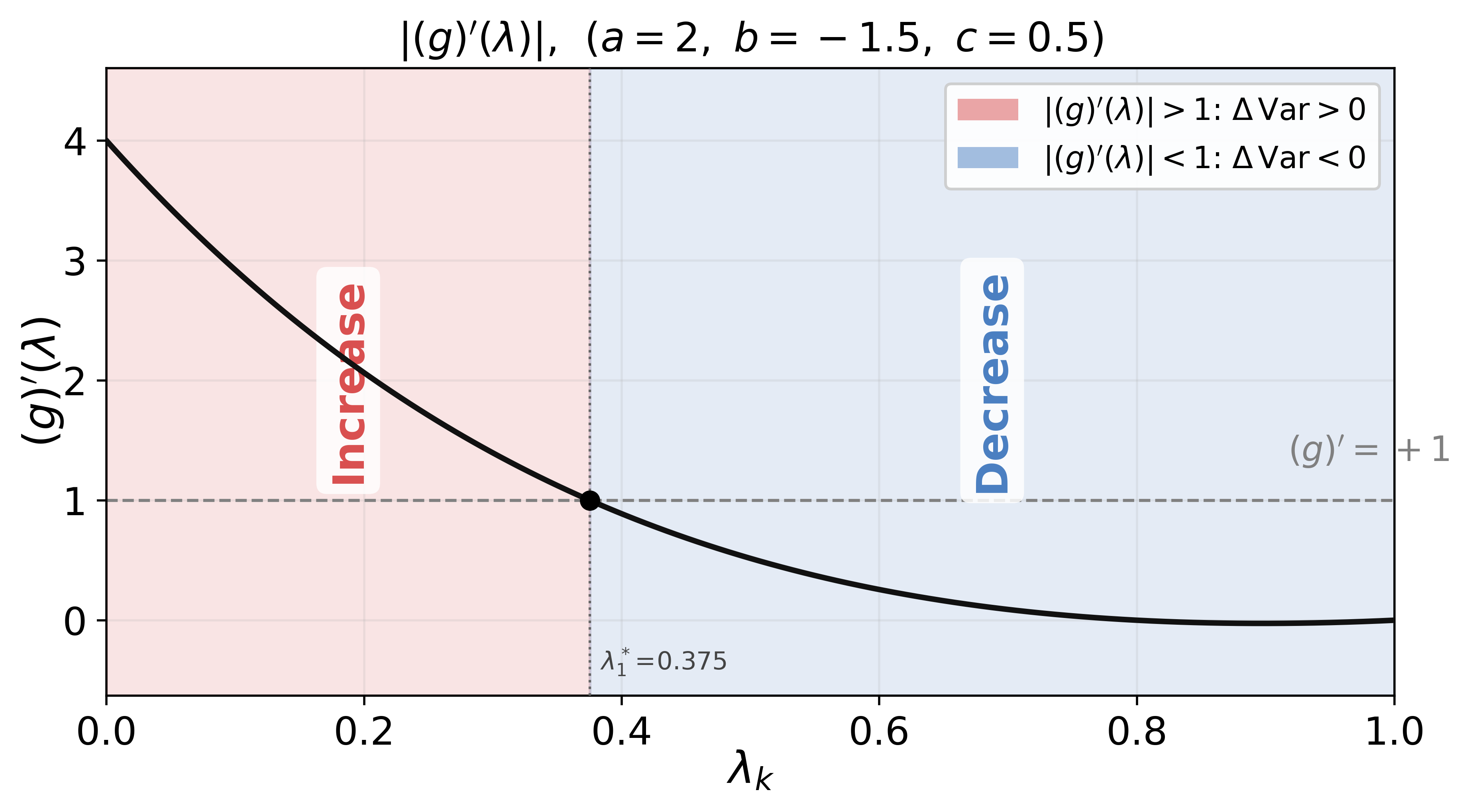}
    \caption{\small Regions of imbalance amplification and reduction. When $\xi_{ij}$ lies in the red (blue) region, the imbalance is amplified (reduced). We also provide figures for other commonly used coefficients in Appendix~\ref{apx:more_coeff}.}
    \label{fig:variance_shift}
    \vspace{-10pt}
\end{wrapfigure}
As an example, consider the commonly used coefficients
\[ 
(a,b,c)=(2,\,-1.5,\;0.5).
\]

Figure~\ref{fig:variance_shift} plots \(g'(\lambda)\) together with the thresholds \(\pm 1\). By Lemma~\ref{lemma:var_diff}, each eigenvalue pair \((\lambda_i,\lambda_j)\) contributes through an intermediate point \(\xi_{ij}\), with sign determined by \(g'(\xi_{ij})^2-1\). Thus, pairs with \(|g'(\xi_{ij})|>1\) increase the variance, while those with \(|g'(\xi_{ij})|<1\) decrease it. The overall variance shift is the sum of these pairwise effects. As the number of Newton--Schulz iterations grows, all \(\lambda_k\) converge to \(1\), eventually entering the right variance-reduction region in Figure~\ref{fig:variance_shift}. Therefore, the variance decreases asymptotically as the number of iterations tends to infinity.

We also investigate the variance trend with respect to the number of Newton--Schulz iterations under real training dynamics in Figure~\ref{fig:var_trend_ns}. As shown in the figure, when the number of Newton--Schulz iterations is insufficient, the variance can be significantly amplified; in particular, for a commonly used choice of 5 iterations, the variance still remains amplified. More importantly, even in the limit of infinitely many iterations, for a non-square layer, the variance along one direction cannot be reduced to zero (e.g., the down-projection layer in Figure~\ref{fig:var_trend_ns}). We will later show that Muon+ can further improve performance exactly in this regime: applying Muon+ along the direction whose variance cannot be eliminated still brings additional gains (see Table~\ref{tab:svd_ablation}).

We further examine variants such as NorMuon\cite{li2025normuon} and PolarExpress\cite{amsel2025polar} under the same realistic pre-training dynamics. Figure~\ref{fig:muon_variants_var} shows that these variants still fail to eliminate the variance imbalance.

\subsection{Imbalanced Updates Hurt Convergence in Muon}
\label{sec:imbalanced_undesired}

An imbalanced update can make optimization less effective. Intuitively, excessively large updates may cause instability, while excessively small updates leave some parts of the network nearly unchanged. We formalize this intuition using a \textit{blockwise descent analysis}~\cite{tomihari2025understanding}.

\begin{theorem}[Imbalance weakens Muon local descent guarantee]
\label{thm:imbalance_bad_informal}
Fix any partition of the $mn$ entries of $\mat W_t$ into $B$ disjoint blocks, and for any vectorization $\mat{x}\in\mathbb R^{mn}$ of matrix $\mat{X} \in \mathbb R^{m \times n}$, let $[\mat{x}]_b$ denote its elements in block $b$. Consider the Muon-style orthogonalized update
\[
\mat{W}_{t+1}=\mat{W}_t-\eta_t\,\mat{O}_t,
\qquad
\mat{O}_t := \mathrm{Ortho}(\nabla L(\mat{W}_t)).
\]
Under local smoothness and a local near-block-diagonal Hessian structure with respect to the chosen partition, the one-step loss decrease satisfies
\[
L(\mat{W}_{t+1})-L(\mat{W}_t)
\;\le\;
-\,\eta_t\,\|\nabla L(\mat{W}_t)\|_{*}
\;+\;\tfrac{\eta_t^{\,2}}{2}\,\widetilde{\Lambda}_{O,t}
\;+\;\mathcal O\!\bigl(\eta_t^{\,3}\,\|\mat{O}_t\|_F^{\,3}\bigr),
\]
\end{theorem}
where $\|\cdot\|_{*}$ denotes the nuclear norm, and $\widetilde{\Lambda}_{O,t}$ is the \emph{block-weighted Hessian curvature} at step $t$:
\begin{equation}
\label{eq:lambda_O_block}
\widetilde{\Lambda}_{O,t}
\;:=\;
\sum_{b=1}^{B}
\big\|[\nabla^2 L(\mat{W}_t)]_b\big\|_2
\;\|[\mat{o}_t]_b\|_2^{\,2},
\qquad
\mat{o}_t := \mathrm{vec}(\mat O_t),
\end{equation}
Prior work~\cite{zhang2024transformers,jiang2023does,zhang2019gradient,crawshaw2022robustness,tomihari2025understanding} has identified a positive \emph{gradient--Hessian correlation}: parameter blocks with larger gradient norms tend to have larger diagonal Hessian blocks. We further find that this correlation is preserved after the polar step in Muon (see Appendix~\ref{apx:update-hessian}). Under this correlation,  larger imbalance in the column gradient norms leads to a larger $\widetilde{\Lambda}_{O,t}$ and hence a smaller largest stable step size. A precise definition and the full proof are provided in Appendix~\ref{apx:imbalance_convergence}. 

The above theoretical analysis is consistent with our experimental observation in Figure~\ref{fig:lr_ablation}: update imbalance reduces the largest stable step size of Muon, whereas our proposed \textsc{Muon+} (which will be explained soon) reduces this imbalance and thus improves training stability and performance.

%\zz{We can show the training curves of a small-size transformer here, demonstrating that Muon starts to become unstable earlier than Muon+ when we increase learning rates.}\rj{Done}
\begin{figure}[t]
    \centering
    \begin{subfigure}[t]{0.48\linewidth}
        \centering
        \includegraphics[width=\linewidth]{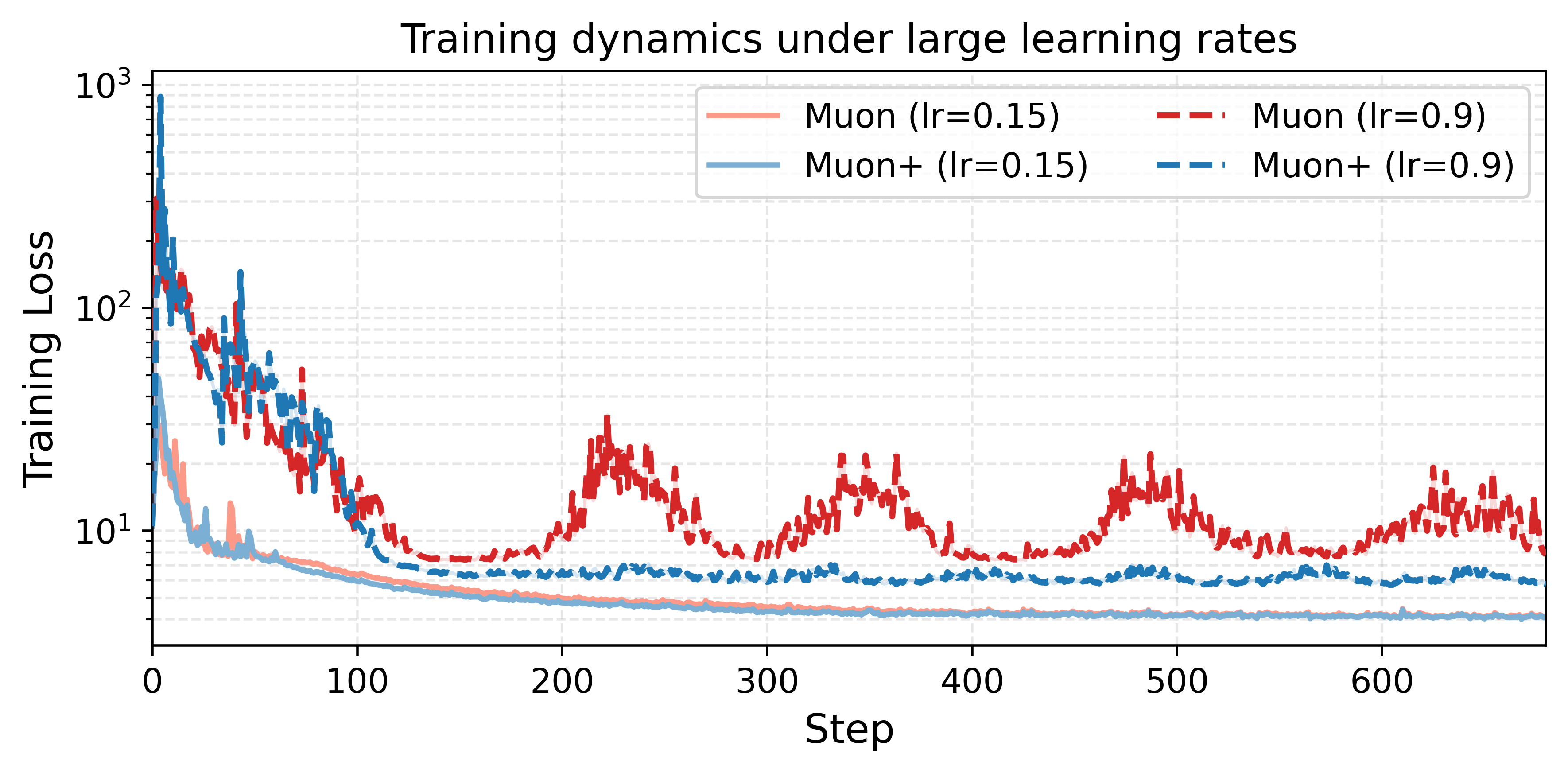}
    \caption{Training behaviors under different learning rates. Muon becomes unstable as the learning rate increases, whereas \textsc{Muon+} remains stable.}
        \label{fig:large_lr}
    \end{subfigure}
    \hfill
    \begin{subfigure}[t]{0.45\linewidth}
        \centering
        \includegraphics[width=\linewidth]{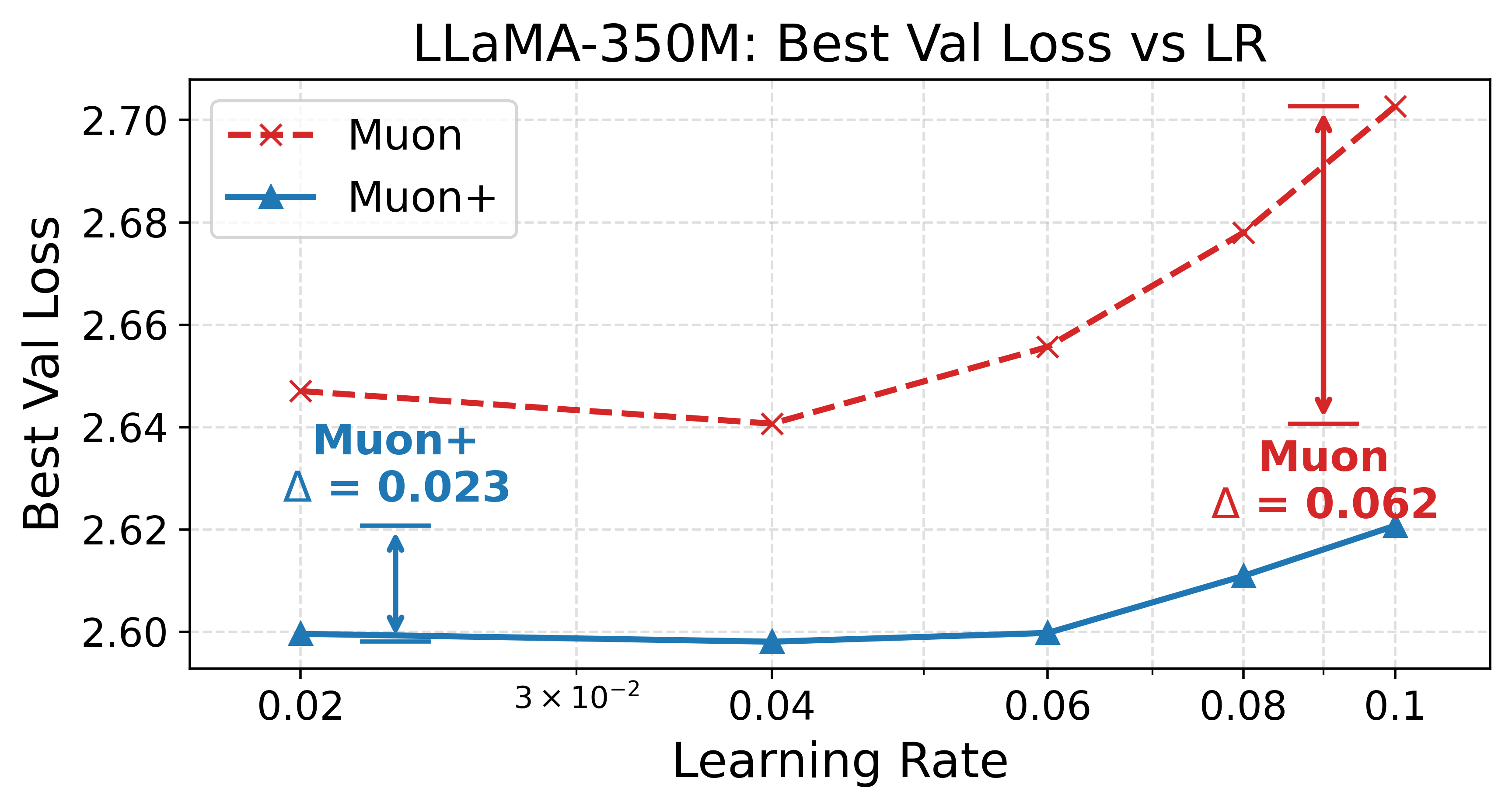}
    \caption{\footnotesize Validation loss of LLaMA 350M. The loss of Muon increases rapidly as the learning rate increases, whereas \textsc{Muon+} remains more stable.}
        \label{fig:llama_val}
    \end{subfigure}
    \caption{Stable step-size comparison between Muon and \textsc{Muon+}. \textsc{Muon+} remains stable under larger learning rates, while Muon exhibits unstable training and degraded validation loss.}
    \label{fig:lr_ablation}
\end{figure}

\subsection{An Easy-to-Deploy Solution: \textsc{Muon+}}
\label{sec:muon+_method}
The above theoretical analysis and experimental demonstration shows that variance imbalance can hurt the performance of Muon. To address this issue, we propose  \textsc{Muon+}, a simple yet highly effective modification by adding a single normalization step after Polar iterations.

\textsc{Muon+} follows the Muon update rule in Eq.~(\ref{eq:muon_update_rule}), while additionally applying a normalization step to the orthogonalized update in order to reduce the imbalance in the update matrix:
\begin{equation}
\begin{aligned}
\mat{M}_t &= \mu \mat{M}_{t-1} + (1 - \mu) \mat{G}_t, \\
\mat{O}_t &= \mathrm{Norm}_{(d)}\!\left(\mathrm{Ortho}(\mat{M}_t)\right), \\
\mat{W}_t &= \mat{W}_{t-1} - \eta \cdot \sqrt{m/n} \cdot \mat{O}_t.
\end{aligned}
\label{eq:muonN_update_rule}
\end{equation}
Here $\mathrm{Norm}_{(d)}(\cdot)$ denotes a normalization operator applied along direction $d$. We provide the pseudocode as in Algorithm~\ref{alg:muon_plus_code} in the Appendix.
We consider column-wise and row-wise normalization, denoted by \(\mathrm{Norm}_{(\mathrm{col})}(\cdot)\) and \(\mathrm{Norm}_{(\mathrm{row})}(\cdot)\), respectively. For \(\mat{X}=[x_{ij}]\in\mathbb{R}^{m\times n}\), define
\[
\mathrm{Norm}_{(\mathrm{col})}(\mat{X}) := \mat{X}\,\mathbf{D}_{\mathrm{col}}^{-1},
\qquad
\mathbf{D}_{\mathrm{col}} := \mathrm{diag}\!\left(
\sqrt{\sum_{i=1}^{m} x_{i1}^{2}},\ldots,\sqrt{\sum_{i=1}^{m} x_{in}^{2}}
\right),
\]
and
\[
\mathrm{Norm}_{(\mathrm{row})}(\mat{X}) := \mathbf{D}_{\mathrm{row}}^{-1}\mat{X},
\qquad
\mathbf{D}_{\mathrm{row}} := \mathrm{diag}\!\left(
\sqrt{\sum_{j=1}^{n} x_{1j}^{2}},\ldots,\sqrt{\sum_{j=1}^{n} x_{mj}^{2}}
\right).
\]
We further consider composed normalizations, such as \(\mathrm{Norm}_{(\mathrm{col\_row})}\) and \(\mathrm{Norm}_{(\mathrm{row\_col})}\), defined by sequential application of the two normalizations.

\section{Experiments}
\label{sec:experiments}

We evaluate \textsc{Muon+} on two widely adopted architectures: GPT and LLaMA.
Our evaluation covers both compute-optimal pre-training and long-horizon overtraining regimes, followed by systematic ablation studies in Section~\ref{sec:ablation}. Note that, all the experiments in this paper use 5 iterations in $\mathrm{Ortho}(\cdot)$ to approximate $\mat{UV}^T$.
\subsection{Pre-training in Compute-Optimal Settings}

\label{exp:GPT_pretraining}
\paragraph{GPT Models.} We first evaluate \textsc{Muon+} on GPT-style models. We pre-train GPT-Small, GPT-Base, and GPT-Large with a compute-optimal~\cite{hoffmann2022training} token-to-parameter (T2P) ratio. All models are trained on the FineWeb dataset~\cite{penedo2024fineweb}, tokenized using the GPT tokenizer, with a vocabulary size of 50,257 and a batch size of 512. Training is conducted on H100/A100 GPUs using mixed precision (bfloat16). Following the setup in~\cite{amsel2025polar}, we apply \textsc{Muon+} (or Muon) to all parameters except embeddings, unembeddings, normalization layers, and positional encodings, which are optimized using AdamW. For the polar operator, we adopt the same configuration as in~\cite{jordan2024muon}. We sweep normalization directions under learning rates in \([0.003, 0.005, 0.01, 0.02, 0.04]\) for both \textsc{Muon+} and Muon except GPT-Huge, and report the best results in Table~\ref{tab:gpt_MuonN_results}. GPT-Huge is trained with a learning rate of {\(0.01\)} using bfloat16 precision (no mixed precision). Detailed hyperparameters and full sweep results are provided in Appendix~\ref{apx:hyperparameter} and Appendix~\ref{apx:sweep}. As shown in Table~\ref{tab:gpt_MuonN_results}, \textsc{Muon+} consistently outperforms Muon across all GPT model scales. 
\begin{table}[h]
\centering
\begin{tabular}{l c c c c}
\toprule
\textbf{Model} & \textbf{Param (M)} & \textbf{Train Tokens (B)} & Muon & \textbf{\textsc{Muon+}} \\
\midrule
GPT-Small & 124 & 3.0  &  29.66& \textbf{27.64}\; \textcolor{green!50!black}  {(-2.02)}\\
GPT-Base  & 362 & 7.2  &  21.70&  \textbf{19.98}\; \textcolor{green!50!black}{(-1.72)}\\
GPT-Large & 774 & 15.5 &  17.82 & \textbf{16.91}\; \textcolor{green!50!black}{(-0.91)}\\
GPT-Huge (bf16) & 6,654 & 20.0 &  15.90 & \textbf{14.69}\; \textcolor{green!50!black}{(-1.21)}\\
\bottomrule
\end{tabular}
\caption{Validation PPL of Muon vs. \textsc{Muon+} on GPT models.}
\label{tab:gpt_MuonN_results}
\end{table}

\paragraph{LLaMA Models.} To extend our evaluation beyond GPT architectures, we benchmark our proposed approach against AdamW and Muon by pre-training LLaMA-based language models. This validation is also conducted on the FineWeb dataset, spanning model capacities from 60M up to 7B parameters (architectural details are provided in Table~\ref{tab:llama_model_configs}).

Based on the compute-optimal scaling guidelines established by \cite{hoffmann2022training}, we strictly pair model sizes with training token budgets: the 60M, 130M, 350M, 1B, and 7B parameter models are trained on 1.1B, 2.2B, 6.4B, 13.1B, and 19.8B tokens, respectively. Across all configurations, we maintain a constant batch size of 512 and employ the LLaMA-2 tokenizer with a 32,000-token vocabulary. In line with the setup described in Section~\ref{exp:GPT_pretraining}, all experiments are executed using mixed precision on H100 and A100 GPUs.

We sweep normalization directions over learning rates in \([0.005, 0.01, 0.02, 0.04, 0.06, 0.08]\) for all models in this section (Section~\ref{sec:ablation_norm_dir}), except LLaMA-1B and LLaMA-7B due to computational constraints. For LLaMA-1B, we sweep only the \textit{col-row} and \textit{row-col} variants. For LLaMA-7B, we use bfloat16 precision and the \textit{row-col} variant with a learning rate of \(0.01\). The results are reported in Table~\ref{tab:llama_pretraining}. Additional hyperparameter details are provided in Appendix~\ref{apx:hyperparameter}. Overall, as shown in Table~\ref{tab:llama_pretraining}, \textsc{Muon+} consistently outperforms the baselines, achieving the best overall performance across all evaluated scales. We also report mean ± standard deviation over 5 random seeds for the GPT-Base and LLaMA-350M models to assess the statistical robustness of MUON+ improvements (see Table~\ref{tab:var_bar}).
\begin{table}[h]
\centering
\setlength{\tabcolsep}{4pt}
\begin{tabular}{l c c c c c}
\toprule
\textbf{Model} & \textbf{Param (M)} & \textbf{Train Tokens (B)} &{AdamW} & Muon & \textbf{\textsc{Muon+}} \\
\midrule
LLaMA-60M  & 58   & 1.1  & 33.10& 25.75 & \textbf{25.25}\; \textcolor{green!50!black}{(-0.50)} \\
LLaMA-130M & 134  & 2.2  & 23.64& 19.06 & \textbf{18.65}\; \textcolor{green!50!black}{(-0.41)} \\
LLaMA-350M & 368  & 6.4  & 16.18& 14.02 & \textbf{13.41}\; \textcolor{green!50!black}{(-0.61)} \\
LLaMA-1B   & 1339 & 13.1 & 14.38& 10.68 &  \textbf{10.31}\; \textcolor{green!50!black}{(-0.37)}\\
LLaMA-7B (bf16)   & 6,738 & 19.8 & -- & 10.54 &  \textbf{10.12}\; \textcolor{green!50!black}{(-0.42)}\\
\bottomrule
\end{tabular}
\caption{Validation PPL of Muon vs. \textsc{Muon+} on LLaMA models.}
\label{tab:llama_pretraining}
\vspace{-15pt}
\end{table}
\subsection{Pre-Training with High Token-to-Parameter Ratios}
\label{sec:overtrain}
\begin{table}[h]
\centering
\begin{tabular}{l c c c c}
\toprule
\textbf{Model} & \textbf{Param (M)} & \textbf{Train Tokens (B)} & Muon & {\bf \textsc{Muon+}} \\
\midrule
GPT-Base   & 362 & 72 & 16.97 & \textbf{15.84}\; \textcolor{green!50!black}{(-1.13)} \\
LLaMA-350M & 368 & 72 & 11.48 & \textbf{11.03}\; \textcolor{green!50!black}{(-0.45)} \\
\bottomrule
\end{tabular}
\caption{Overtraining GPT-Base/LLaMA-350M. Models are trained on 72 billion FineWeb tokens.}
\label{tab:overtrain}
\end{table}

\begin{table}[t]
  \centering
  \vspace{-15pt}
    \begin{tabular}{ccccccccc}
    \toprule
          & Ave. & OBQA & HellaSwag & ARC-E & WSC   & Winogrande & BoolQ & PIQA \\
    \midrule
    Muon  & 0.481 & 0.306 & 0.446 & 0.444 & \textbf{0.375} & 0.522 & 0.574 & 0.700 \\
    \midrule
    \textsc{Muon+} & \textbf{0.494} & \textbf{0.320} & \textbf{0.480} & \textbf{0.471} & 0.365 & \textbf{0.534} & \textbf{0.577} & \textbf{0.712} \\
    \bottomrule
    \end{tabular}%
    \caption{Downstream evaluation results on GPT-Base overtraining checkpoint.}
  \label{tab:gpt_downstram}%
  \vspace{-15pt}
\end{table}%

\begin{figure}[t]
    \centering
        \centering
        \includegraphics[width=\linewidth]{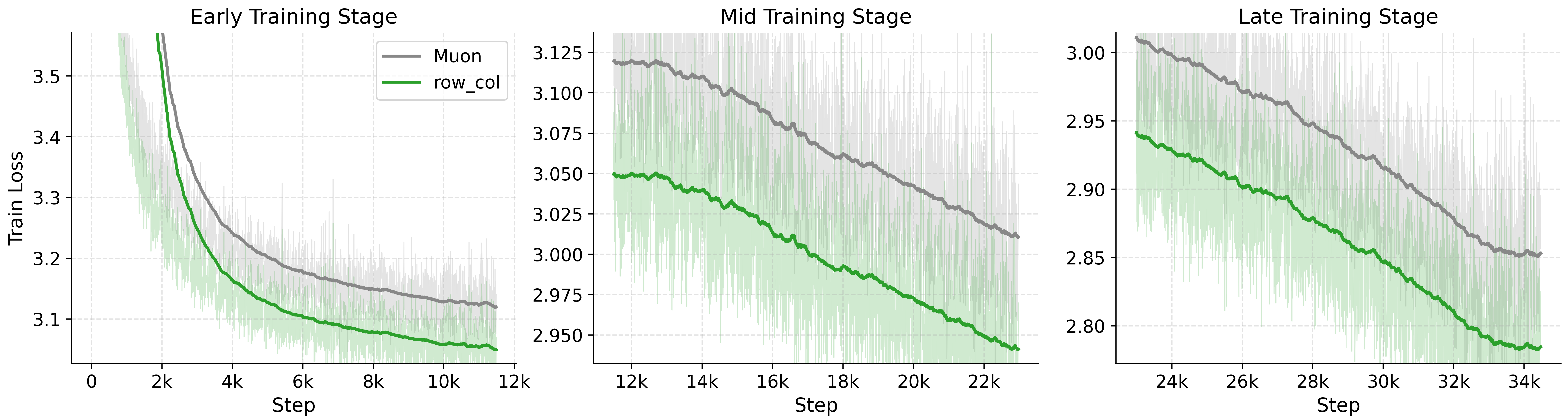}
    \caption{Training loss curves under overtraining for GPT-Base.}
    \vspace{-10pt}
    \label{fig:overtrain_loss}
\end{figure}
We further study \textsc{Muon+} in the overtraining regime using GPT-Base and LLaMA-350M. Both models are trained on 72 billion FineWeb tokens, corresponding to a token-to-parameter ratio of approximately 200. Detailed hyperparameters are provided in Appendix~\ref{apx:hyperparameter}.

Table~\ref{tab:overtrain} shows that \textsc{Muon+} outperforms Muon for both models in the overtraining regime. The improvement remains at a token-to-parameter ratio of approximately 200, indicating that the additional normalization continues to be effective under extended training. Figure~\ref{fig:overtrain_loss} shows that the gap is stable throughout training.

\paragraph{Downstream Evaluations.} Furthermore, we compare \textsc{Muon+} with Muon on downstream tasks using the pretrained models. We evaluate their commonsense reasoning performance, with results reported in Table~\ref{tab:gpt_downstram}. \textsc{Muon+} outperforms Muon on most tasks and achieves higher average accuracy.

\subsection{Training Efficiency}
\label{sec:complexity}

\textsc{Muon+} has nearly the same per-step runtime and memory cost as Muon (see Appendix~\ref{tab:per_step_cost}). In this section, we report the wall-clock time required to reach the same target loss, where the target is set to the final training loss of the Muon baseline. We use the same data scale as in Table~\ref{tab:gpt_MuonN_results} and Table~\ref{tab:llama_pretraining}. The results are shown in Table~\ref{tab:norm_overhead}. \textsc{Muon+} speeds up pre-training by up to \(37.1\%\) in the evaluated settings.

\begin{table}[h]
\centering
\setlength{\tabcolsep}{3pt}
\begin{tabular}{lrrrr}
\toprule
\textbf{Model} & \textbf{Target loss} & \textbf{Muon steps} & \textbf{Muon+ steps} & \textbf{Speed-up} \\
\midrule
LLaMA-130M & 3.060 & 4228 & 3448 & \textcolor{green!50!black}{$\uparrow$22.6\%} \\
LLaMA-350M & 2.788 & 3064 & 2374 & \textcolor{green!50!black}{$\uparrow$29.1\%} \\
\midrule
GPT-Base   & 3.144 & 3447 & 2515 & \textcolor{green!50!black}{$\uparrow$37.1\%} \\
GPT-Large  & 2.927 & 3710 & 3032 & \textcolor{green!50!black}{$\uparrow$22.4\%} \\
\bottomrule
\end{tabular}
\caption{Training cost required to reach a target loss.}
\label{tab:norm_overhead}
\end{table}

\subsection{Ablation Study}
\label{sec:ablation}
%\zz{Please organize the ablation study as one sub-section in the result section.}\rj{}
To better understand the source of performance gains, 
we conduct systematic ablations on the key design choices of \textsc{Muon+}.
In particular, we analyze the effects normalization directions
and orthogonalization methods while keeping all other training settings fixed.

\paragraph{Impact of Normalization Directions.}
\label{sec:ablation_norm_dir}

\begin{table}[H]
\centering
\setlength{\tabcolsep}{6pt}
\begin{tabular}{lccccc}
\toprule
Model & \textbf{None} & \textbf{Col} & \textbf{Row} & \textbf{Col-Row} & \textbf{Row-Col} \\
\midrule
LLaMA-60M  & 25.75 & 25.34 & 25.29 & \textbf{25.25} & \textbf{25.25} \\
LLaMA-130M & 19.34 & 19.16 & 18.98 & \textbf{18.65} & 18.68 \\
LLaMA-350M & 14.02 & 13.73 & 13.46 & \textbf{13.41} & 13.44 \\
\bottomrule
\end{tabular}
\caption{Best validation perplexity under different normalization directions.
For each model, we report the best result across learning rates.
Lower is better.}
\label{tab:norm_best}
\vspace{-15pt}
\end{table}

We study the impact of different normalization directions, including \texttt{none}(Muon Baseline), \texttt{col}, \texttt{row}, \texttt{col\_row}, and \texttt{row\_col}, across multiple model scales, see Appendix~\ref{apx:sweep} for more quantitative results. As shown in Table~\ref{tab:norm_best}, introducing normalization leads to better optimization behavior compared to the baseline Muon across all evaluated settings. This improvement maintains as the model scale increases. Importantly, applying bi-directional normalization consistently outperforms single-directional normalization.

\paragraph{Ablation for Polar Methods \textbf{$\mathrm{Ortho}(\cdot)$}.} To validate the robustness of \textsc{Muon+} under different polar functions, we adopt 3 different methods (all with 5 iterations) in this section: You \cite{cesistasqueezing}, Jordan\cite{jordan2024muon} and the more recent PolarExpress\cite{amsel2025polar}. Hyperparameters are identical as in Table~\ref{tab:hparams_llama}. As shown in Table~\ref{tab:polar_ablation}, \textsc{Muon+} consistently outperforms Muon across all evaluated methods. As discussed in Section~\ref{sec:amp}, one direction may remain unbalanced for a non-square matrix even when $\mat{U}\mat{V}^\top$ is computed exactly via SVD. Therefore, applying \textsc{Muon+} on top of this exact orthogonal update can still yield additional gains. This is experimentally shown in Table~\ref{tab:svd_ablation}.

\begin{table}[H]
\centering
\small

\begin{minipage}[t]{0.48\linewidth}
\centering
\footnotesize
\setlength{\tabcolsep}{5pt}
\begin{tabular}{c c c c}
\toprule
\textbf{Model} & $\mathrm{Ortho}(\cdot)$  & \textbf{Muon} & \textbf{\textsc{Muon+}} \\
\midrule
\multirow{3}{*}{LLaMA-350M}
& You & 14.01 & \textbf{13.38}\textcolor{green!50!black}{(-0.63)} \\
& Jordan & 14.02 & \textbf{13.41}\textcolor{green!50!black}{(-0.61)} \\
& PolarExpress & 13.90 & \textbf{13.27}\textcolor{green!50!black}{(-0.63)} \\
\bottomrule
\end{tabular}
\caption{Validation perplexity comparison among Muon, and \textsc{Muon+}(\texttt{col\_row}) on LLaMA-350M; all methods use 5 iterations.}
\label{tab:polar_ablation}
\end{minipage}
\hfill
\begin{minipage}[t]{0.48\linewidth}
\centering
\footnotesize
\setlength{\tabcolsep}{5pt}
\begin{tabular}{c c c c}
\toprule
\textbf{Model} & $\mathrm{Ortho}(\cdot)$ & \textbf{Muon} & \textbf{\textsc{Muon+}} \\
\midrule
GPT-Small & SVD & 29.17 & \textbf{27.90}\textcolor{green!50!black}{(-1.27)} \\
\bottomrule
\end{tabular}
\caption{\textsc{Muon+} and Muon with exact SVD. Learning rate is 0.005 for all runs. \textsc{Muon+} uses \texttt{col\_row} as its normalization direction.}
\label{tab:svd_ablation}
\end{minipage}

\end{table}

\section{Related Work and Comparison with \textsc{Muon+}}
\label{sec:related_work_comparison}

%We also compare our method with prior works~\cite{boissin2025turbo,li2025normuon,si2025adamuon}. For all experiments involving Muon variants, we sweep the learning rate from 0.003 to 0.04, except for AdaMuon, for which we sweep from 0.001 to 0.005. The results are reported in Table~\ref{tab:comparison-results}. \textsc{Muon+} achieves the best overall performance. Notably, unlike NorMuon~\cite{li2025normuon} and AdaMuon~\cite{si2025adamuon}, \textsc{Muon+} does not require storing any additional optimizer states, incurs no memory overhead.

Muon~\citep{jordan2024muon} has shown promising performance in large-scale pre-training. Follow-up work can be grouped into three directions.

\begin{itemize}[leftmargin=*]
    \item {\bf Optimizing the polar function (complementary to Muon+).}
    Several methods improve Muon by designing better polar iterations. You~\cite{cesistasqueezing} learns per-iteration coefficients, \textsc{PolarExpress}~\cite{amsel2025polar} derives minimax-optimal Newton--Schulz coefficients, and \textsc{Turbo-Muon}~\cite{boissin2025turbo} manipulates the spectrum for faster convergence of the polar iteration. Because Muon+ does not change anything inside the polar functions, these methods are complementary to \textsc{Muon+}. As shown in Table~\ref{tab:polar_ablation}, combining \textsc{Muon+} with these methods always yields additional performance improvements.

    \item {\bf Pre-conditioning techniques (potential competitors with Muon+).}
    Some recent techniques improve Muon with additional preconditioning. {NorMuon}~\cite{li2025normuon} adds a per-neuron second moment, {AdaMuon}~\cite{si2025adamuon} combines Muon with an Adam-style element-wise step, and {Mano}~\cite{gu2026mano} projects updates onto the oblique manifold. These methods are closer competitors to \textsc{Muon+}, as they also aim to improve the update geometry. As shown in Table~\ref{tab:muon_comparison_results}, \textsc{Muon+} consistently outperforms these competing baselines.\footnote{For all experiments in this comparison except AdaMuon, we swept the learning rate over \([0.005, 0.01, 0.02, 0.04, 0.06, 0.08]\). For AdaMuon, which requires smaller learning rates, we swept \([0.005, 0.003, 0.001]\). We report the best result for each method.}. Notably, unlike NorMuon and AdaMuon, \textsc{Muon+} does not require additional optimizer states.
    \item {\bf Theoretical foundations of Muon (relevant but not competing with Muon+).}
    Existing theoretical studies mainly explain Muon from spectral or non-Euclidean optimization views. \cite{bernstein2024modular} frames Muon as steepest descent under the spectral norm; \cite{kovalev2025understanding} analyzes Muon through a non-Euclidean trust-region formulation; \cite{su2025isotropic} models Muon using an isotropic curvature approximation. These results do not capture the properties of the parameter-space. We provide an alternative perspective: this work explains why and how Muon can amplify the parameter-space imbalance (Lemma~\ref{lemma:var_diff}, Corollary~\ref{cor:positive_shift_exists}) and why such imbalance provably hurts convergence (Theorem~\ref{thm:imbalance_bad_informal}). Our analysis also helps explain why preconditioning methods like NorMuon and AdaMuon can be effective, as their preconditioning mechanisms implicitly reduce the parameter-space imbalance induced by Muon (see Figure~\ref{fig:muon_variants_var}).
\end{itemize}
\begin{table}[h]
\centering
\setlength{\tabcolsep}{3pt}
\begin{tabular}{lcc}
\toprule
\textbf{Method} & \textbf{GPT-Small} & \textbf{GPT-Base} \\
\midrule
\textbf{Muon+} & \textbf{27.64} & \textbf{19.98} \\
Turbo-Muon & 29.69 & 21.91 \\
NorMuon    & 28.44 & 21.31 \\
AdaMuon    & 29.27 & 22.38 \\
\bottomrule
\end{tabular}
\caption{\small Validation PPL ($\downarrow$) across different Muon methods.}
\label{tab:muon_comparison_results}
\vspace{-10pt}
\end{table}
\vspace{-10pt}
\section{Conclusion}
In this work, we have shown that polar iterations can amplify the imbalance in the \emph{parameter space} of the update in Muon. Through a blockwise descent analysis, we have shown that such an imbalance reduces Muon's largest stable step size. To address this issue, we have proposed \textsc{Muon+}, a simple yet effective modification that adds a single normalization step after orthogonalization. Across GPT and LLaMA pre-training from 60M to 7B parameters, \textsc{Muon+} has consistently improved pre-training performance with zero memory overhead and significant overall training speedup.
\clearpage
\bibliography{ref}
\bibliographystyle{abbrv}
\clearpage
\appendix
\section{Blockwise Descent Analysis for Muon}
\label{apx:imbalance_convergence}
This appendix follows the blockwise descent framework of~\cite{tomihari2025understanding}. To align with the notation of prior work, we identify the matrix parameter $\mat W\in\mathbb R^{m\times n}$ with a vectorization
\[
\theta:=\mathrm{vec}(\mat W)\in\mathbb R^{P},
\qquad
P=mn,
\qquad
\|\theta\|_2=\|\mat W\|_F.
\]
Throughout this appendix, we fix an arbitrary partition of the index set $\{1,\dots,P\}$ into $B$ non-overlapping blocks
\[
\{1,\dots,P\}
=
I_1\sqcup I_2\sqcup\cdots\sqcup I_B,
\qquad
P_b:=|I_b|.
\]
For any $x\in\mathbb R^{P}$ we write $[x]_b:=x\big|_{I_b}\in\mathbb R^{P_b}$, so $\|x\|_2^2=\sum_{b=1}^B \|[x]_b\|_2^2$. Two canonical choices (the ones used in Section~\ref{sec:amp}) are:
\begin{itemize}
\item \emph{Column blocks}: $B=n$, $P_b=m$, and $[\theta]_b=\mat W_{:,b}$ (i.e., $\theta$ is the column-stacking vectorization and each block is a column of $\mat W$);
\item \emph{Row blocks}: $B=m$, $P_b=n$, and $[\theta]_b=\mat W_{b,:}^\top$ (i.e., $\theta$ is the row-stacking vectorization and each block is a row of $\mat W$).
\end{itemize}

\subsection{Block-Hessian and weighted curvature}
\label{apx:vec_convention}

For each $b\in\{1,\dots,B\}$, let $[\nabla^2 L(\theta)]_b\in\mathbb R^{P_b\times P_b}$ denote the $b$-th principal submatrix of the Hessian under the chosen partition, and define the block-diagonal part of the Hessian by
\[
\nabla^2 L_D(\theta)
:=
\mathrm{blockdiag}\!\bigl(\{[\nabla^2 L(\theta)]_b\}_{b=1}^B\bigr).
\]

For the Muon-style orthogonalized update
\[
\mat O_t \;:=\; \mathrm{Ortho}(\nabla L(\mat W_t)),
\qquad
o_t \;:=\; \mathrm{vec}(\mat O_t)\in\mathbb R^{P},
\]
we define the update-weighted block curvature at iterate $\theta_t$ by
\begin{equation}
\label{eq:lambda_O_block_appendix}
\widetilde\Lambda_{O,t}
:=
\sum_{b=1}^{B}
\bigl\|[\nabla^2 L(\theta_t)]_b\bigr\|_2
\,\|[o_t]_b\|_2^2.
\end{equation}
Since $\mat O_t$ is semi-orthogonal (if $\nabla L(\mat W_t)=\mat U\mat\Sigma\mat V^\top$, then $\mat O_t=\mat U\mat V^\top$), we have
\[
\|o_t\|_2^2 \;=\; \|\mat O_t\|_F^2 \;=\; \min(m,n) \;=:\; r.
\]

\subsection{Local assumptions}
\label{apx:local_assumptions}

We use the following local analogues of Assumptions~4.2 and~4.3 of~\cite{tomihari2025understanding}.

\begin{assumption}[Local Hessian Lipschitz]
\label{asm:local_lipschitz}
There exist a neighborhood $\mathcal N_t\subseteq\mathbb R^{P}$ of $\theta_t$, containing the line segment
\[
\{\theta_t+s(\theta_{t+1}-\theta_t):\,s\in[0,1]\},
\]
and a constant $\rho_H>0$ such that
\[
\|\nabla^2 L(\theta)-\nabla^2 L(\theta')\|_2
\le
\rho_H\|\theta-\theta'\|_2
\qquad
\text{for all }\theta,\theta'\in\mathcal N_t.
\]
\end{assumption}

\begin{assumption}[Local near block-diagonal Hessian]
\label{asm:local_blockdiag}
With respect to the chosen block partition $\{I_b\}_{b=1}^B$, there exists $\delta_D\ge 0$ such that, for all $\theta\in\mathcal N_t$,
\[
\|\nabla^2 L(\theta)-\nabla^2 L_D(\theta)\|_2
\le
\delta_D.
\]
\end{assumption}

\subsection{Technical lemma}
\label{apx:lemma}
\begin{lemma}
\label{lem:taylor_hl}
Under Assumption~\ref{asm:local_lipschitz}, for any $\theta,\theta'\in\mathcal N_t$,
\begin{equation}
\label{eq:taylor_hl}
L(\theta')
\le
L(\theta)
+
\nabla L(\theta)^\top(\theta'-\theta)
+
\frac12(\theta'-\theta)^\top \nabla^2 L(\theta)(\theta'-\theta)
+
\frac{\rho_H}{6}\|\theta'-\theta\|_2^3.
\end{equation}
\end{lemma}

\begin{proof}
Define
\[
\nu(s):=\theta+s(\theta'-\theta),\qquad s\in[0,1].
\]
By the fundamental theorem of calculus,
\[
L(\theta')-L(\theta)
=
\int_0^1 \nabla L(\nu(s))^\top(\theta'-\theta)\,ds.
\]
Also,
\[
\nabla L(\nu(s))-\nabla L(\theta)
=
\int_0^s \nabla^2 L(\nu(r))(\theta'-\theta)\,dr.
\]
Substituting this identity yields
\[
L(\theta')-L(\theta)
=
\nabla L(\theta)^\top(\theta'-\theta)
+
\int_0^1\int_0^s
(\theta'-\theta)^\top \nabla^2 L(\nu(r))(\theta'-\theta)\,dr\,ds.
\]
Add and subtract $\nabla^2 L(\theta)$ inside the quadratic form. The main term becomes
\[
\int_0^1\int_0^s
(\theta'-\theta)^\top \nabla^2 L(\theta)(\theta'-\theta)\,dr\,ds
=
\frac12(\theta'-\theta)^\top \nabla^2 L(\theta)(\theta'-\theta).
\]
For the remainder, Assumption~\ref{asm:local_lipschitz} gives
\[
\|\nabla^2 L(\nu(r))-\nabla^2 L(\theta)\|_2
\le
\rho_H\|\nu(r)-\theta\|_2
=
\rho_H r\|\theta'-\theta\|_2.
\]
Hence
\[
\bigl|
(\theta'-\theta)^\top
(\nabla^2 L(\nu(r))-\nabla^2 L(\theta))
(\theta'-\theta)
\bigr|
\le
\rho_H r\|\theta'-\theta\|_2^3.
\]
Integrating over $0\le r\le s\le 1$ gives
\[
\int_0^1\int_0^s \rho_H r\|\theta'-\theta\|_2^3\,dr\,ds
=
\frac{\rho_H}{6}\|\theta'-\theta\|_2^3.
\]
Combining the terms proves~\eqref{eq:taylor_hl}.
\end{proof}

\subsection{One-step descent bound: formal statement and proof}
\label{apx:onestep}
\begin{theorem}[Theorem~\ref{thm:imbalance_bad_informal} is restated]
\label{thm:formal_onestep}
Under Assumptions~\ref{asm:local_lipschitz} and~\ref{asm:local_blockdiag}, the Muon-style orthogonalized update
\[
\theta_{t+1}=\theta_t-\eta_t\,o_t,
\qquad
o_t := \mathrm{vec}\bigl(\mathrm{Ortho}(\nabla L(\mat W_t))\bigr),
\qquad
\eta_t>0,
\]
satisfies
\begin{equation}
\label{eq:onestep_descent}
L(\theta_{t+1})
\le
L(\theta_t)
-\eta_t\|\nabla L(\mat W_t)\|_{*}
+
\frac{\eta_t^2}{2}\bigl(\widetilde\Lambda_{O,t}+\delta_D\,\|o_t\|_2^2\bigr)
+
\frac{\eta_t^3\rho_H}{6}\|o_t\|_2^3,
\end{equation}
where $\|\cdot\|_{*}$ denotes the nuclear norm and $\|o_t\|_2^2=\|\mat O_t\|_F^2=\min(m,n)=r$.
\end{theorem}

\textbf{Remark.} For simplicity, Theorem~\ref{thm:imbalance_bad_informal} is stated under the assumption $\delta_D \to 0$, which does not affect the conclusion.

\begin{proof}
Let
\[
g_t:=\nabla L(\theta_t),
\qquad
\mat O_t := \mathrm{Ortho}(\nabla L(\mat W_t)),
\qquad
o_t := \mathrm{vec}(\mat O_t).
\]
Applying Lemma~\ref{lem:taylor_hl} with $\theta=\theta_t$ and $\theta'=\theta_{t+1}=\theta_t-\eta_t\, o_t$, we obtain
\begin{equation}
\label{eq:muon_taylor_raw}
L(\theta_{t+1})
\le
L(\theta_t)
-\eta_t\, g_t^\top o_t
+
\frac{\eta_t^2}{2}\, o_t^\top \nabla^2 L(\theta_t)\, o_t
+
\frac{\eta_t^3\rho_H}{6}\|o_t\|_2^3.
\end{equation}
We handle the first- and second-order terms separately.

\medskip
\noindent\textit{Step 1: First-order term.} Write the (thin) SVD of $\nabla L(\mat W_t)$ as
\[
\nabla L(\mat W_t) = \mat U\mat\Sigma\mat V^\top,
\]
so that $\mat O_t = \mat U\mat V^\top$. Since $g_t^\top o_t$ is the Euclidean inner product of the vectorizations of $\nabla L(\mat W_t)$ and $\mat O_t$ under a common vectorization convention, it equals the matrix Frobenius inner product, so
\[
g_t^\top o_t
=
\bigl\langle \nabla L(\mat W_t),\,\mat O_t\bigr\rangle_F
=
\mathrm{tr}\bigl(\mat V\mat\Sigma^\top\mat U^\top\mat U\mat V^\top\bigr)
=
\mathrm{tr}(\mat\Sigma)
=
\|\nabla L(\mat W_t)\|_{*},
\]
\medskip
\noindent\textit{Step 2: Second-order term.} By Assumption~\ref{asm:local_blockdiag},
\[
\|\nabla^2 L(\theta_t)-\nabla^2 L_D(\theta_t)\|_2\le \delta_D.
\]
Therefore,
\[
o_t^\top \nabla^2 L(\theta_t) o_t
=
o_t^\top \nabla^2 L_D(\theta_t) o_t
+
o_t^\top\bigl(\nabla^2 L(\theta_t)-\nabla^2 L_D(\theta_t)\bigr)o_t,
\]
and the second term is bounded by
\[
\bigl|
o_t^\top\bigl(\nabla^2 L(\theta_t)-\nabla^2 L_D(\theta_t)\bigr)o_t
\bigr|
\le
\delta_D\|o_t\|_2^2.
\]
Hence
\begin{equation}
\label{eq:quad_split}
o_t^\top \nabla^2 L(\theta_t) o_t
\le
o_t^\top \nabla^2 L_D(\theta_t) o_t
+
\delta_D\|o_t\|_2^2.
\end{equation}

Since $\nabla^2 L_D(\theta_t)$ is block-diagonal with respect to the chosen partition, its quadratic form decomposes blockwise:
\[
o_t^\top \nabla^2 L_D(\theta_t) o_t
=
\sum_{b=1}^B
[o_t]_b^\top [\nabla^2 L(\theta_t)]_b [o_t]_b.
\]
For each block, the operator norm gives
\[
[o_t]_b^\top [\nabla^2 L(\theta_t)]_b [o_t]_b
\le
\bigl\|[\nabla^2 L(\theta_t)]_b\bigr\|_2\,\|[o_t]_b\|_2^2.
\]
Summing over $b$, we obtain
\[
o_t^\top \nabla^2 L_D(\theta_t) o_t
\le
\sum_{b=1}^{B}
\bigl\|[\nabla^2 L(\theta_t)]_b\bigr\|_2
\,\|[o_t]_b\|_2^2
=
\widetilde\Lambda_{O,t}.
\]
Combining this with~\eqref{eq:quad_split} gives
\[
o_t^\top \nabla^2 L(\theta_t) o_t
\le
\widetilde\Lambda_{O,t}+\delta_D\,\|o_t\|_2^2.
\]
Substituting the first-order identity from Step~1 and this bound into~\eqref{eq:muon_taylor_raw} proves~\eqref{eq:onestep_descent}.
\end{proof}
Using $\|o_t\|_2=\|\mat O_t\|_F$, Theorem~\ref{thm:formal_onestep} is equivalently the matrix statement
\[
L(\mat W_{t+1})
\le
L(\mat W_t)
-\eta_t\|\nabla L(\mat W_t)\|_{*}
+
\frac{\eta_t^2}{2}\bigl(\widetilde\Lambda_{O,t}+\delta_D\,\|\mat O_t\|_F^2\bigr)
+
\frac{\eta_t^3\rho_H}{6}\|\mat O_t\|_F^3,
\]
where $\mat O_t=\mathrm{Ortho}(\nabla L(\mat W_t))$ and $\|\mat O_t\|_F^2=\min(m,n)$. Instantiating the partition $\{I_b\}_{b=1}^B$ as either the columns or the rows of $\mat W_t$ recovers the column- and row-block versions used in Section~\ref{sec:amp}: the principal Hessian blocks $[\nabla^2 L(\theta_t)]_b$ and the absolute squared block norms $\|[o_t]_b\|_2^2$ in~\eqref{eq:lambda_O_block_appendix} then specialize to column- or row-indexed quantities, respectively, while $\delta_D$ is the near-block-diagonal slack of the Hessian under the chosen partition.
\clearpage

\subsection{Positive Correlation Before and After the Polar Step}
\label{apx:update-hessian}
\begin{figure}
    \centering
    \includegraphics[width=0.8\linewidth]{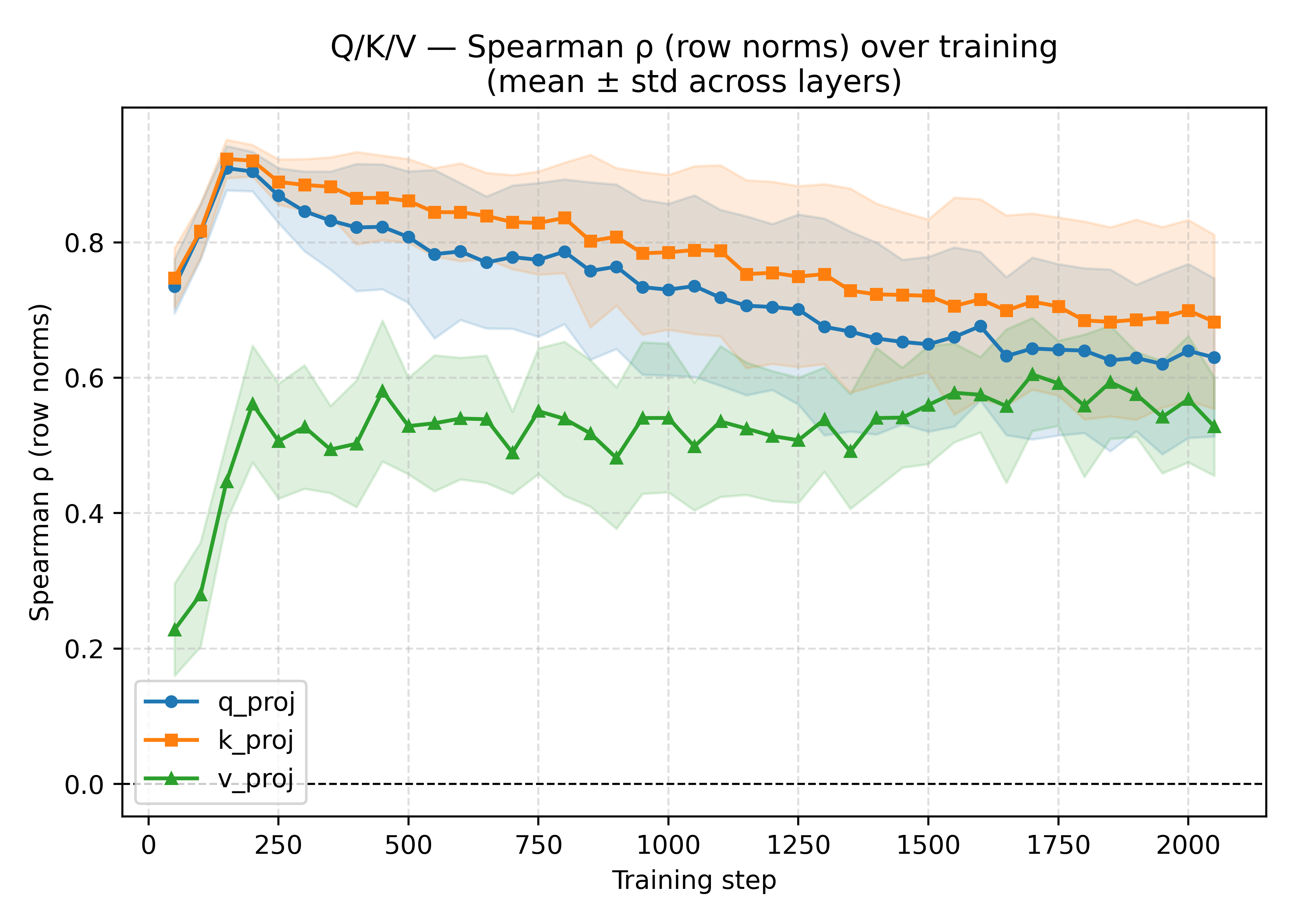}
    \caption{Spearman rank correlation between pre-polar and post-polar (using 5 NS iterations) row norms over training (mean $\pm$ standard deviation across layers). All correlations remain positive throughout training, indicating that rows that are large before the polar step tend to remain large afterward.}
    \label{fig:qkv_spearman}
\end{figure}

Prior work has identified a positive \emph{gradient--Hessian correlation}: parameter blocks with larger gradient norms tend to be associated with larger diagonal Hessian blocks~\cite{zhang2024transformers,jiang2023does,zhang2019gradient,crawshaw2022robustness,tomihari2025understanding}. Motivated by this observation, we empirically examine whether the polar step preserves the \emph{relative blockwise magnitudes} of the pre-polar update. If it does, then the blocks that are large before the polar step will still tend to remain large afterward.

To quantify this effect, for each matrix updated by Muon, let
\[
g_j := \|[\mat G]_{j,:}\|_2,
\qquad
q_j := \|[\mat Q]_{j,:}\|_2,
\]
denote the row norms of the pre-polar input $\mat G$ and the post-polar output $\mat Q$, respectively. We then compute the Spearman rank correlation~\cite{spearman1961proof} between the two sequences:
\begin{equation}
\label{eq:spearman_def}
\rho_{\mathrm{sp}}(g,q)
:=
\mathrm{Corr}\!\bigl(\mathrm{rank}(g_1,\dots,g_n),\,
\mathrm{rank}(q_1,\dots,q_n)\bigr).
\end{equation}
A positive value of $\rho_{\mathrm{sp}}$ indicates that rows that are relatively large before the polar step also tend to remain relatively large afterward; a larger value indicates stronger preservation of the blockwise ordering.

Figure~\ref{fig:qkv_spearman} reports the Spearman correlation over training, using row norms and averaging across layers. We observe that the correlations remain consistently positive throughout training. Moreover, the correlations are generally substantial. This shows that the polar step does not destroy the relative blockwise ordering of the pre-polar update.
\clearpage

\section{Existence of Positive Expected Variance Shift}
\label{apx:amp_proof}
\subsection{Proof of Lemma~\ref{lemma:var_diff}}

\begin{proof}
We divide the proof into four steps. Throughout, let $\lambda_k := \sigma_k^2$ and $g(\lambda) := \varphi(\sqrt{\lambda})^2 = \lambda(a + b\lambda + c\lambda^2)^2$.

\medskip
\noindent\textit{Step 1: Row energies in terms of $\ \mat{U}$ and the spectrum.}

Since $\mat{M} = \mat{U}\mat{\Sigma}\mat{V}^\top$ and $\mat{Q} = \mat{U}\,\varphi(\mat{\Sigma})\,\mat{V}^\top$, the row energies $s_i(\cdot) = \|\cdot_{i:}\|_2^2$ are
\[
  s_i(\mat{M}) = \sum_{k=1}^{n} \lambda_k\, u_{ik}^2, \qquad
  s_i(\mat{Q}) = \sum_{k=1}^{n} g(\lambda_k)\, u_{ik}^2.
\]

\medskip
\noindent\textit{Step 2: Expected variance under Haar measure.}

We compute $\mathbb{E}[\mathrm{Var}(s(\mat{M}))]$ and $\mathbb{E}[\mathrm{Var}(s(\mat{Q}))]$ simultaneously; both have the form $s_i = \sum_k \alpha_k\, u_{ik}^2$ with $\alpha_k = \lambda_k$ or $\alpha_k = g(\lambda_k)$, respectively. Assume $\mat{U}$ is Haar-distributed on $O(n)$, each row $(u_{i1},\dots,u_{in})$ is uniform on $S^{n-1}$. The relevant spherical moments are
\[
  \mathbb{E}[u_{ik}^2] = \frac{1}{n}, \qquad
  \mathbb{E}[u_{ik}^4] = \frac{3}{n(n+2)}, \qquad
  \mathbb{E}[u_{ik}^2 u_{i\ell}^2] = \frac{1}{n(n+2)} \;\; (k \neq \ell),
\]
which follow from exchangeability of coordinates and the constraint $\sum_k u_{ik}^2 = 1$.

The first moment of the row energy is $\mathbb{E}[s_i] = \frac{1}{n}\sum_k \alpha_k$, independent of $i$. For the second moment,
\begin{align*}
  \mathbb{E}[s_i^2]
  &= \sum_{k} \alpha_k^2\,\mathbb{E}[u_{ik}^4] + \sum_{k \neq \ell} \alpha_k \alpha_\ell\,\mathbb{E}[u_{ik}^2 u_{i\ell}^2] \\
  &= \frac{3}{n(n+2)} \sum_{k} \alpha_k^2 + \frac{1}{n(n+2)} \sum_{k \neq \ell} \alpha_k \alpha_\ell \\
  &= \frac{1}{n(n+2)} \Bigl( 2\sum_{k} \alpha_k^2 + \Bigl(\sum_{k} \alpha_k\Bigr)^{\!2} \Bigr),
\end{align*}
where we used $3\sum_k \alpha_k^2 + \sum_{k\neq \ell} \alpha_k\alpha_\ell = 2\sum_k \alpha_k^2 + (\sum_k \alpha_k)^2$. Note that $\mathbb{E}[s_i^2]$ is also independent of $i$ by exchangeability of coordinates under Haar measure. Moreover, the total row energy is deterministic: by orthogonality of $\mat{U}$,
\[
  \sum_{i=1}^{n} s_i = \sum_{k=1}^{n} \alpha_k \sum_{i=1}^{n} u_{ik}^2 = \sum_{k=1}^{n} \alpha_k,
\]
which does not depend on $\mat{U}$. Therefore $\bar{s} := \frac{1}{n}\sum_i s_i$ is a deterministic constant, and $\mathbb{E}[\bar{s}^{\,2}] = \bar{s}^{\,2} = \bigl(\mathbb{E}[s_i]\bigr)^2$. It follows that
\begin{align*}
  \mathbb{E}[\mathrm{Var}(s)]
  &= \mathbb{E}[s_i^2] - \bigl(\mathbb{E}[s_i]\bigr)^2
  = \frac{1}{n(n+2)} \Bigl( 2\sum_{k} \alpha_k^2 + \Bigl(\sum_k \alpha_k\Bigr)^{\!2} \Bigr) - \frac{1}{n^2}\Bigl(\sum_k \alpha_k\Bigr)^{\!2} \\
  &= \frac{2}{n^2(n+2)} \Bigl( n\sum_{k} \alpha_k^2 - \Bigl(\sum_k \alpha_k\Bigr)^{\!2} \Bigr).
\end{align*}
Setting $\alpha_k = g(\lambda_k)$ and $\alpha_k = \lambda_k$ and taking the difference,
\begin{equation}\label{eq:var_shift_raw}
  \mathbb{E}\bigl[\mathrm{Var}(s(\mat{Q})) - \mathrm{Var}(s(\mat{M}))\bigr]
  = \frac{2}{n^2(n+2)} \Bigl[
    n\sum_{k} g(\lambda_k)^2 - \Bigl(\sum_{k} g(\lambda_k)\Bigr)^{\!2}
    - n\sum_{k} \lambda_k^2 + \Bigl(\sum_{k} \lambda_k\Bigr)^{\!2}
  \Bigr].
\end{equation}

\medskip
\noindent\textit{Step 3: Reformulation via pairwise differences.}

Applying the identity $\;n\sum_{k} x_k^2 - (\sum_{k} x_k)^{2} = \sum_{i<j} (x_i - x_j)^2\;$ separately to $x_k = g(\lambda_k)$ and $x_k = \lambda_k$, we rewrite \eqref{eq:var_shift_raw} as
\begin{equation}\label{eq:var_shift_pairwise}
  \mathbb{E}\bigl[\mathrm{Var}(s(\mat{Q})) - \mathrm{Var}(s(\mat{M}))\bigr]
  = \frac{2}{n^2(n+2)} \sum_{1 \le i < j \le n}
    \bigl[(g(\lambda_i) - g(\lambda_j))^2 - (\lambda_i - \lambda_j)^2\bigr].
\end{equation}

\medskip
\noindent\textit{Step 4: Applying the mean value theorem.}

For each pair with $\lambda_i \neq \lambda_j$, the mean value theorem gives a point $\xi_{ij} \in \bigl(\min\{\lambda_i,\lambda_j\},\, \max\{\lambda_i,\lambda_j\}\bigr)$ such that
\[
  g(\lambda_i) - g(\lambda_j) = g'(\xi_{ij})\,(\lambda_i - \lambda_j).
\]
Squaring and substituting into \eqref{eq:var_shift_pairwise},
\[
  (g(\lambda_i) - g(\lambda_j))^2 - (\lambda_i - \lambda_j)^2
  = \bigl(g'(\xi_{ij})^2 - 1\bigr)(\lambda_i - \lambda_j)^2.
\]
When $\lambda_i = \lambda_j$ both sides vanish trivially. Therefore
\[
  \mathbb{E}\bigl[\mathrm{Var}(s(\mat{Q})) - \mathrm{Var}(s(\mat{M}))\bigr]
  = \frac{2}{n^2(n+2)} \sum_{1 \le i < j \le n}
    (\lambda_i - \lambda_j)^2 \bigl(g'(\xi_{ij})^2 - 1\bigr),
\]
which completes the proof.
\end{proof}

\subsection{Proof of Corollary~\ref{cor:positive_shift_exists}}
\begin{proof}
We divide the proof into three steps.

\textit{Step 1: \(a>1\) for Newton-Schulz iteration in Muon}

In Muon, the input matrix is first Frobenius-normalized. Hence its singular values satisfy
\[
0\le \sigma_k \le 1,
\qquad k=1,\dots,n,
\]
and therefore
\[
0\le \lambda_k=\sigma_k^2 \le 1.
\]
Thus the relevant regime for the Newton--Schulz map is the interval \([0,1]\). 
Moreover, in Muon the Newton--Schulz step is designed to move singular values toward the target value \(1\), namely,
\begin{equation}
\label{eq:push_up}
\varphi(x)>x
\qquad
\text{for all }x\in(0,1).
\end{equation}

\medskip
\noindent
Since
\[
\varphi(x)=x(a+bx^2+cx^4)=ax+bx^3+cx^5,
\]
we have
\[
\varphi(x)-x=(a-1)x+bx^3+cx^5
= x\bigl((a-1)+bx^2+cx^4\bigr).
\]
Assume for contradiction that \(a\le 1\).

If \(a<1\), then
\[
(a-1)+bx^2+cx^4 < 0
\]
for all sufficiently small \(x>0\), because the constant term \(a-1\) is strictly negative and the higher-order terms vanish as \(x\to0\). Hence
\[
\varphi(x)-x<0
\]
for all sufficiently small \(x>0\), contradicting \eqref{eq:push_up}.
Without appealing to any specific choice of Newton--Schulz coefficients, we assume:
\[
\varphi'(0)=a>1.
\]

\medskip
\noindent
\textit{Step 2: \(a>1\) implies that \(|g'(\lambda)|>1\) on a neighborhood of \(0\).}

By definition,
\[
g(\lambda)=\lambda(a+b\lambda+c\lambda^2)^2.
\]
Differentiating gives
\[
g'(\lambda)
=
(a+b\lambda+c\lambda^2)^2
+
2\lambda(a+b\lambda+c\lambda^2)(b+2c\lambda).
\]
In particular,
\[
g'(0)=a^2.
\]
Since \(a>1\), we get
\[
g'(0)=a^2>1.
\]
Because \(g'\) is continuous, there exists \(\delta>0\) such that
\[
g'(\lambda)^2>1
\qquad
\text{for all }\lambda\in[0,\delta].
\]

\medskip
\noindent
\textit{Step 3: Positivity of the variance shift.}

By Lemma~\ref{lemma:var_diff}, for each pair \(1\le i<j\le n\), there exists
\[
\xi_{ij}\in(\min\{\lambda_i,\lambda_j\},\,\max\{\lambda_i,\lambda_j\})
\]
such that
\[
\mathbb E\!\left[\Var(s(\mat Q))-\Var(s(\mat M))\right]
=
\frac{2}{n^2(n+2)}
\sum_{1\le i<j\le n}
(\lambda_i-\lambda_j)^2\bigl(g'(\xi_{ij})^2-1\bigr).
\]
Now suppose that
\[
\lambda_1,\dots,\lambda_n\in[0,\delta].
\]
Then for every \(i<j\),
\[
\xi_{ij}\in[0,\delta],
\]
and therefore
\[
g'(\xi_{ij})^2-1>0.
\]
Hence each summand
\[
(\lambda_i-\lambda_j)^2\bigl(g'(\xi_{ij})^2-1\bigr)
\]
is nonnegative. Since the \(\lambda_k\) are not all equal, there exists at least one pair \(i<j\) such that
\[
(\lambda_i-\lambda_j)^2>0,
\]
and for that pair the corresponding summand is strictly positive. Therefore the whole sum is strictly positive, and since
\[
\frac{2}{n^2(n+2)}>0,
\]
we conclude that
\[
\mathbb E\!\left[\Var(s(\mat Q))-\Var(s(\mat M))\right]>0.
\]

This proves the claim.
\end{proof}

\clearpage
\section{Pseudocode of Muon+}
\begin{algorithm}[htb]
\caption{Python code for the \textsc{Muon+} update.}
\label{alg:muon_plus_code}
\begin{lstlisting}[language=Python]
def muon_plus_step(W, M_prev, G, mu, lr, d="col", eps=1e-8):
    # momentum
    M = mu * M_prev + (1.0 - mu) * G
    # orthogonalize
    U = Ortho(M)  # newton-schulz
    # normalize
    O = norm_dir(U, d=d, eps=eps)
    # update
    m, n = W.shape[-2], W.shape[-1]
    W = W - lr * (m / n) ** 0.5 * O
    return W, M
def norm_dir(X, d="col", eps=1e-8):
    if d == "col":
        denom = (X.square().sum(dim=-2, keepdim=True) + eps).sqrt()
        return X / denom
    if d == "row":
        denom = (X.square().sum(dim=-1, keepdim=True) + eps).sqrt()
        return X / denom
    if d == "col_row":
        return norm_dir(norm_dir(X, "col", eps), "row", eps)
    if d == "row_col":
        return norm_dir(norm_dir(X, "row", eps), "col", eps)
\end{lstlisting}
\end{algorithm}

\clearpage
\newpage
\section{Reigions of Imbalance Amplification and Reduction}
\label{apx:more_coeff}
\subsection[Jordan Coefficients]{Jordan Coefficients~\cite{jordan2024muon}}
In~\cite{jordan2024muon}, the coefficients are set to
\[
(a,b,c)=(3.4445,\,-4.7750,\;2.0315).
\]
The resulting regions are as Figure~\ref{fig:jordan_shift}:
\begin{figure}[H]
    \centering
    \includegraphics[width=0.5\linewidth]{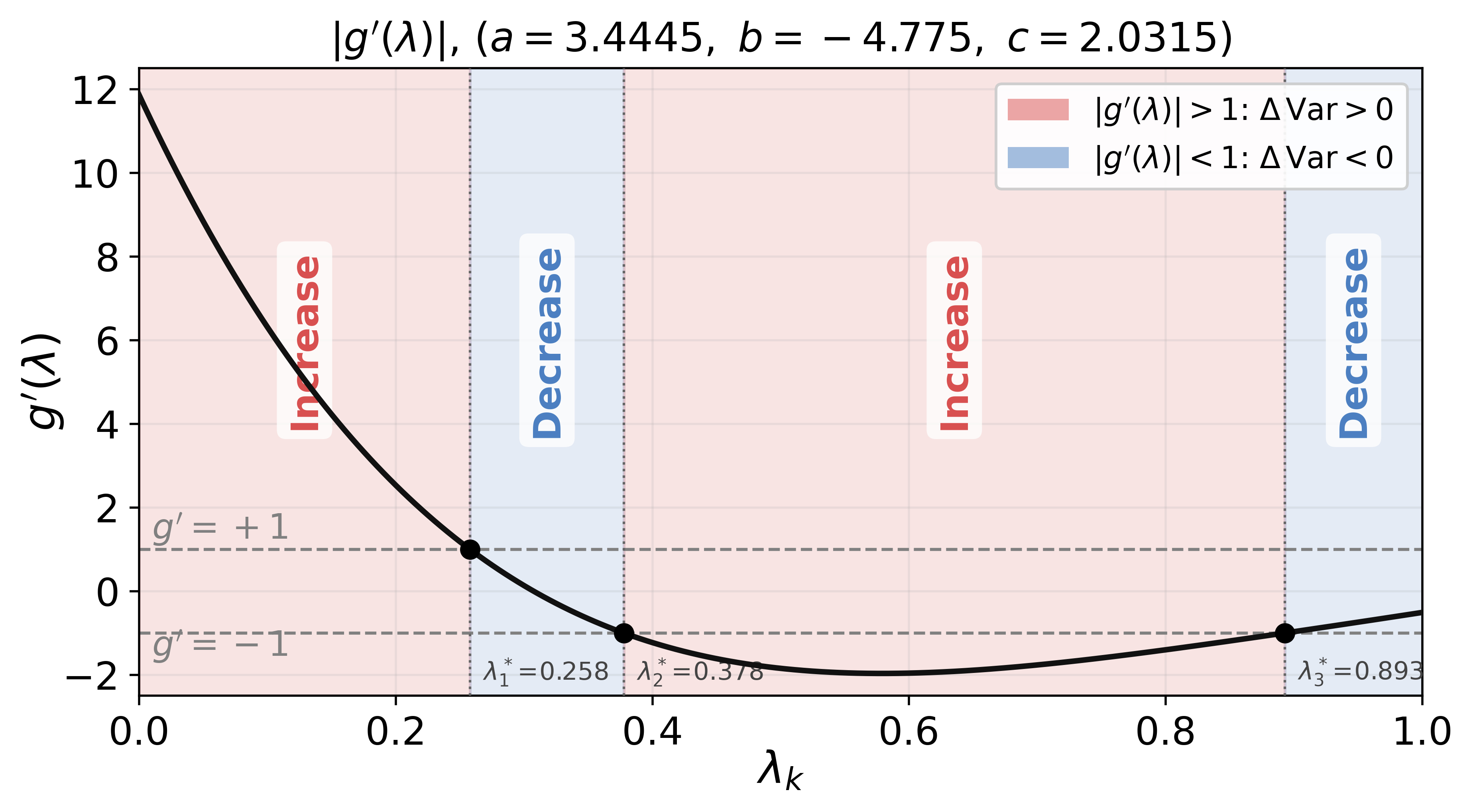}
    \caption{Regions of imbalance amplification and reduction for Jordan Coefficients.}
    \label{fig:jordan_shift}
\end{figure}

\subsection[You Coefficients]{You Coefficients~\cite{cesistasqueezing}}
You~\cite{cesistasqueezing} further optimizes the coefficients at each iteration:
\[
\begin{aligned}
\mathrm{coeff}=\big[
&(3955/1024,\,-8306/1024,\,5008/1024),\\
&(3735/1024,\,-6681/1024,\,3463/1024),\\
&(3799/1024,\,-6499/1024,\,3211/1024),\\
&(4019/1024,\,-6385/1024,\,2906/1024),\\
&(2677/1024,\,-3029/1024,\,1162/1024),\\
&(2172/1024,\,-1833/1024,\,682/1024)
\big].
\end{aligned}
\]
Then the regions are as follows per step:
\begin{figure}[H]
    \centering
    \includegraphics[width=\linewidth]{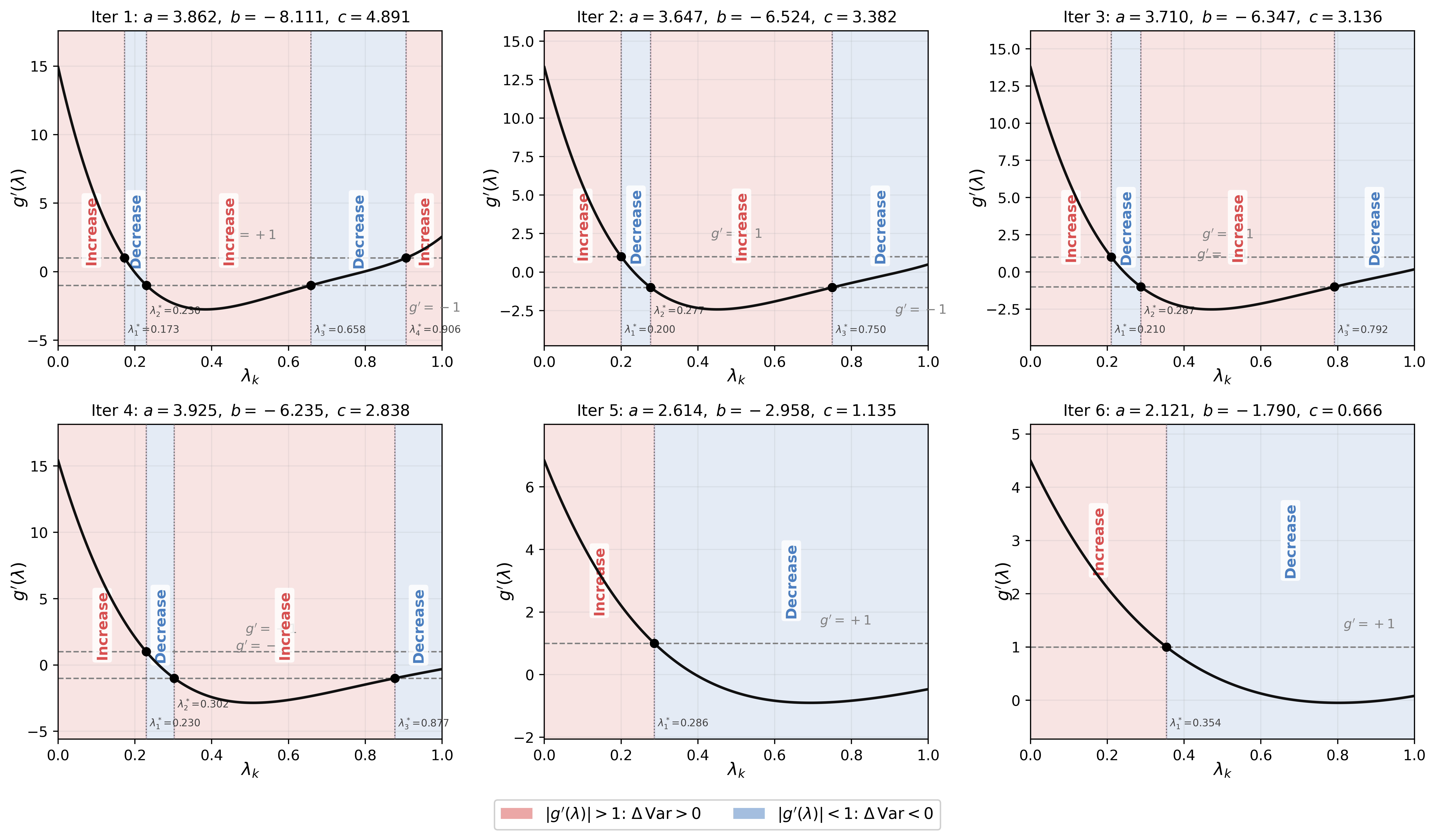}
    \caption{Regions of imbalance amplification and reduction for You Coefficients.}
    \label{fig:you_shift}
\end{figure}

\subsection[PolarExpress Coefficients]{PolarExpress Coefficients~\cite{amsel2025polar}}
\cite{amsel2025polar} uses iteration-dependent coefficients:
\[
\begin{aligned}
\{(a_t,b_t,c_t)\}_{t=1}^8=\big[
&(8.28721201814563,\,-23.595886519098837,\,17.300387312530933),\\
&(4.107059111542203,\,-2.9478499167379106,\,0.5448431082926601),\\
&(3.9486908534822946,\,-2.908902115962949,\,0.5518191394370137),\\
&(3.3184196573706015,\,-2.488488024314874,\,0.51004894012372),\\
&(2.300652019954817,\,-1.6689039845747493,\,0.4188073119525673),\\
&(1.891301407787398,\,-1.2679958271945868,\,0.37680408948524835),\\
&(1.8750014808534479,\,-1.2500016453999487,\,0.3750001645474248),\\
&(1.875,\,-1.25,\,0.375)
\big],
\end{aligned}
\]
with subsequent coefficients numerically equal to \((1.875,\,-1.25,\,0.375)\).
\begin{figure}[H]
    \centering
    \includegraphics[width=\linewidth]{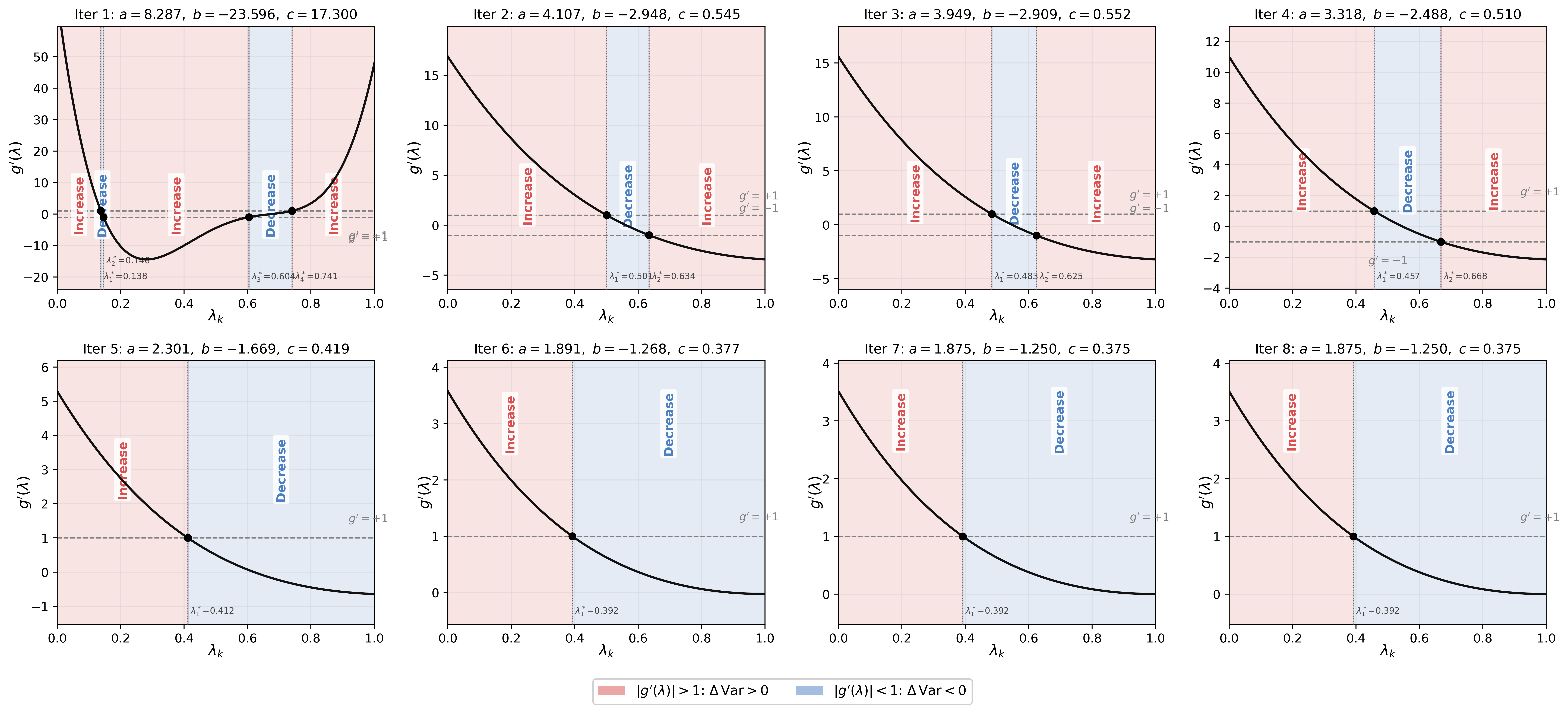}
    \caption{Regions of imbalance amplification and reduction for PolarExpress Coefficients.}
    \label{fig:PolarExpress_shift}
\end{figure}
\clearpage
\section{Hyperparameter}
\label{apx:hyperparameter}
\subsection{Model Configurations}
\begin{table}[h]
\centering
\begin{tabular}{lcccc}
\toprule
\textbf{Model} & $n_{\text{embd}}$ & $n_{\text{layer}}$ & $n_{\text{head}}$ & Param(M) \\
\midrule
GPT-Small & 768  & 12 & 12 & 124 \\
GPT-Base & 1024 & 24 & 16 & 362 \\
GPT-Large & 1280 & 36 & 20 & 774\\
GPT-Huge & 4096 & 32 & 32 & 6654 \\
\bottomrule
\end{tabular}
\caption{Architecture configurations of GPT models.}
\label{tab:gpt_model_configs}
\end{table}

\begin{table}[h]
\centering
\begin{tabular}{lccccc}
\toprule
\textbf{Model} & $n_{\text{embd}}$ & $n_{\text{layer}}$ & $n_{\text{head}}$ & FFN dim & Param(M) \\
\midrule
60M  & 512  & 8  & 8  & 1376 & 58 \\
130M & 768  & 12 & 12 & 2048 & 134 \\
350M & 1024 & 24 & 16 & 2736 & 368 \\
1B   & 2048 & 24 & 32 & 5461 & 1339 \\
7B   & 4096 & 32 & 32 & 11008 & 6738 \\
\bottomrule
\end{tabular}
\caption{Architecture configurations of LLaMA-style models.}
\label{tab:llama_model_configs}
\end{table}

\subsection{Training Configurations}
\subsubsection{GPT Models}
\begin{table}[htb]
\centering
\begin{tabular}{l lccc}
\toprule
\textbf{Model} & \textbf{Hyperparameter} & Muon & \textsc{Muon+} \\
\midrule

\multirow{4}{*}{GPT-Small}
& Sequence length  & 2048 & 2048 \\
& Learning rate    & $0.005$ & $0.01$ \\
& Weight decay     & $0.1$ & $0.1$ \\
& Norm$_d$         & $-$ & Col-Row \\

\midrule

\multirow{4}{*}{GPT-Base}
& Sequence length  & 4096 & 4096 \\
& Learning rate    & $0.005$ & $0.005$ \\
& Weight decay     & $0.1$ & $0.1$ \\
& Norm$_d$         & $-$ & Row \\

\midrule

\multirow{4}{*}{GPT-Large}
& Sequence length  & 8192 & 8192 \\
& Learning rate    & $0.02$ & $0.01$ \\
& Weight decay     & $0.1$ & $0.1$ \\
& Norm$_d$         & $-$ & Row-Col \\

\midrule

\multirow{4}{*}{GPT-Huge}
& Sequence length  & 1024 & 1024 \\
& Learning rate    & $0.01$ & $0.01$ \\
& Weight decay     & $0.005$ & $0.005$ \\
& Norm$_d$         & $-$ & Row-Col \\

\bottomrule
\end{tabular}
\caption{Best training hyperparameters for GPT-Small/Base/Large/Huge on FineWeb. Sequence lengths are set to 2048/4096/8192 for Small/Base/Large, respectively. We use Jordan orthogonalization for all runs. We keep the same learning rate scheduler as in NanoGPT: a constant learning rate for the first 40\% of training steps followed by a linear decay to zero. Sweeping results are provided in Appendix~\ref{apx:sweep}.}
\label{tab:hparams_gpt_all}
\end{table}

\begin{table}[h]
\centering
\begin{tabular}{l lccc}
\toprule
\textbf{Model} & \textbf{Hyperparameter} & Muon & \textsc{Muon+} \\
\midrule

\multirow{3}{*}{GPT-Base}
& Learning rate    & $0.005$ & $0.005$ \\
& Weight decay     & $0.1$ & $0.1$ \\
& Norm$_d$         & $-$ & Row-Col \\

\bottomrule
\end{tabular}
\caption{Training hyperparameters for overtraining GPT-Base. We train 72 billion FineWeb tokens for each setting. Sequence length is 4096 for both runs.}
\label{tab:hparams_gpt_base_overtrain}
\end{table}

\clearpage

\subsubsection{LLaMA Models}

\begin{table}[htb]
\centering
\begin{tabular}{l lccc}
\toprule
\textbf{Model} & \textbf{Hyperparameter} & Muon & \textsc{Muon+} \\
\midrule

\multirow{6}{*}{LLaMA-60M}
& Sequence length  & 1024 & 1024 \\
& Learning rate    & $0.06$ & $0.06$ \\
& LR scheduler     & Cosine & Cosine \\
& Weight decay     & $0.1$ & $0.1$ \\
& Warmup ratio     & $0.1$ & $0.1$ \\
& Norm$_d$         & $-$ & Row-Col \\

\midrule

\multirow{6}{*}{LLaMA-130M}
& Sequence length  & 1024 & 1024 \\
& Learning rate    & $0.02$ & $0.02$ \\
& LR scheduler     & Cosine & Cosine \\
& Weight decay     & $0.1$ & $0.1$ \\
& Warmup ratio     & $0.1$ & $0.1$ \\
& Norm$_d$         & $-$ & Col-Row \\

\midrule

\multirow{6}{*}{LLaMA-350M}
& Sequence length  & 4096 & 4096 \\
& Learning rate    & $0.04$ & $0.04$ \\
& LR scheduler     & Cosine & Cosine \\
& Weight decay     & $0.1$ & $0.1$ \\
& Warmup ratio     & $0.1$ & $0.1$ \\
& Norm$_d$         & $-$ & Col-Row \\

\midrule

\multirow{6}{*}{LLaMA-1B}
& Sequence length  & 4096 & 4096 \\
& Learning rate    & $0.02$ & $0.02$ \\
& LR scheduler     & Cosine & Cosine \\
& Weight decay     & $0.1$ & $0.1$ \\
& Warmup ratio     & $0.1$ & $0.1$ \\
& Norm$_d$         & $-$ & Row-Col \\

\midrule

\multirow{6}{*}{LLaMA-7B}
& Sequence length  & 1024 & 1024 \\
& Learning rate    & $0.01$ & $0.01$ \\
& LR scheduler     & Cosine & Cosine \\
& Weight decay     & $0.005$ & $0.005$ \\
& Warmup ratio     & $0.1$ & $0.1$ \\
& Norm$_d$         & $-$ & Row-Col \\

\bottomrule
\end{tabular}
\caption{Best training hyperparameters for LLaMA on FineWeb. Sequence lengths are set to 1024 for 60M/130M and 4096 for 350M/1B. All runs use Jordan orthogonalization. $\mathrm{Norm}_{(\mathrm{col\_row})}$ and $\mathrm{Norm}_{(\mathrm{row\_col})}$ yield nearly identical performance. Additional sweeping results are provided in Appendix~\ref{apx:sweep}.}
\label{tab:hparams_llama}
\end{table}
\begin{table}[h]
\centering
\begin{tabular}{l lccc}
\toprule
\textbf{Model} & \textbf{Hyperparameter} & Muon & \textsc{Muon+} \\
\midrule

\multirow{3}{*}{LLaMA-350M}
& Learning rate    & $0.04$ & $0.04$ \\
& Weight decay     & $0.1$ & $0.1$ \\
& Norm$_d$         & $-$ & Col-Row \\

\bottomrule
\end{tabular}
\caption{Training hyperparameters for overtraining LLaMA-350M. We train 72 billion FineWeb tokens for each setting. Sequence length is 4096 for both runs.}
\label{tab:hparams_llama_base_overtrain}
\end{table}

%%%%%%%%%%%%%%%%%
\clearpage
\newpage
\section{Detailed Sweep Experiments}
\label{apx:sweep}

\begin{table*}[h]
  \centering
  \setlength{\tabcolsep}{5pt}
  \begin{tabular}{lccccc}
    \toprule
    LR & \textbf{None} & \textbf{Col} & \textbf{Row} & \textbf{Col-Row} & \textbf{Row-Col} \\
    \midrule
    \multicolumn{6}{l}{\textit{LLaMA-60M}} \\
    \midrule
    0.005 & 28.98 & 28.17 & 28.14 &28.10  & \textbf{28.03} \\
    0.01  & 27.03 & 26.59 & 26.42 &26.46  & \textbf{26.41} \\
    0.02  & 26.22 & 25.83 & 25.69 &\textbf{25.65}  & 25.68 \\
    0.04  & 25.90 & 25.49 & 25.30 &25.28  & \textbf{25.28} \\
    0.06  & 25.75 & 25.34 & 25.29 &25.25  & \textbf{25.25} \\
    0.08  & 25.97 & 25.39 & 25.41 &25.29  & \textbf{25.28} \\
    \midrule
    \multicolumn{6}{l}{\textit{LLaMA-130M}} \\
    \midrule
    0.005 & 20.40 & 20.05 & 19.89 & \textbf{19.86} & 19.87 \\
    0.01  & 19.34 & 19.16 & \textbf{18.98} & 18.98 & 19.01 \\
    0.02  & 19.06 & 18.85 & 18.67 & \textbf{18.65} & 18.68 \\
    0.04  & 19.35 & 19.00 & 18.92 & \textbf{18.87} & 18.87 \\
    0.06  & 19.70 & 19.31 & 19.25 & 19.21 & \textbf{19.20} \\
    0.08  & 20.14 & 19.63 & 19.61 & 19.56 & \textbf{19.54} \\
    \midrule
    \multicolumn{6}{l}{\textit{LLaMA-350M}} \\
    \midrule
    0.005 & 15.54 & 14.97 & 14.94 & 14.92 & \textbf{14.86} \\
    0.01  & 14.48 & 14.18 & 14.00 & 13.96 & \textbf{13.95} \\
    0.02  & 14.11 & 13.73 & 13.46 & \textbf{13.43} & 13.46 \\
    0.04  & 14.02 & 13.61 & 13.44 & \textbf{13.41} & 13.44 \\
    0.06  & 14.23 & 13.65 & 13.48 & 13.50 & \textbf{13.46} \\
    0.08  & 14.56 & 13.76 & \textbf{13.60} & 13.60 & 13.61 \\
    \midrule
    \multicolumn{6}{l}{\textit{LLaMA-1B}} \\
    \midrule
    0.005 & 11.47 & - & - & 10.95 & \textbf{10.92} \\
    0.01  & 10.88 & - & - & 10.50 & \textbf{10.48} \\
    0.02  & 10.68 & - & - & 10.32 & \textbf{10.31} \\
    0.04  & 10.74 & - & - & \textbf{10.42} & 10.43 \\
    0.06  & 10.91 & - & - & \textbf{10.57} & 10.61 \\
    0.08  & 11.15 & - & - & \textbf{10.77} & 10.78 \\
    \bottomrule
  \end{tabular}
  \caption{Best validation perplexity per norm setting
           and learning rate for all LLaMA models.
           Bold marks the best entry in each row.
           \textbf{None} denotes the Muon baseline.}
  \label{tab:llama_all_norm_lr}
\end{table*}

\begin{table}[h]
  \centering
  \setlength{\tabcolsep}{6pt}
  \begin{tabular}{lccccc}
    \toprule
    LR & \textbf{None} & \textbf{Col} & \textbf{Row} & \textbf{Col-Row} & \textbf{Row-Col} \\
    \midrule
    \multicolumn{6}{l}{\textit{GPT-Small}} \\
    \midrule
    0.003 & 31.10 & 29.75 & \textbf{28.37} & 28.39 & 29.01 \\
    0.005 & 29.66 & 28.81 & 27.97 & \textbf{27.91} & 28.09 \\
    0.01  & 29.87 & 28.08 & 27.69 & \textbf{27.64} & 27.76 \\
    0.02  & 30.13 & 29.03 & 29.27 & 28.56 & \textbf{27.72} \\
    0.04  & 30.14 & 29.71 & 30.25 & 30.58 & \textbf{29.37} \\
    \midrule
    \multicolumn{6}{l}{\textit{GPT-Base}} \\
    \midrule
    0.003 & 22.10 & 21.37 & 20.22 & \textbf{20.20} & 20.75 \\
    0.005 & 21.70 & 20.63 & \textbf{19.98} & 20.03 & 20.11 \\
    0.01  & 21.89 & 20.73 & 20.93 & 20.74 & \textbf{20.35} \\
    0.02  & 21.75 & 20.83 & 21.15 & 21.10 & \textbf{20.27} \\
    0.04  & {22.14} & 23.27 & 23.16 & 23.41 & \textbf{22.12} \\
    \midrule
    \multicolumn{6}{l}{\textit{GPT-Large}} \\
    \midrule
    0.005 & 18.26  & - & - &  -& \textbf{17.16}\\
    0.01  & 18.28 & - &- &-  & \textbf{16.91} \\
    0.02  & 17.82 & - & - &-  & \textbf{17.52}\\
    0.04  & 18.84 & - & - &- & \textbf{17.83}\\
    \bottomrule
  \end{tabular}
  \caption{Best validation perplexity per norm setting
           and learning rate for GPT Models. Bold marks the best entry in each row.
           \textbf{None} is the Muon baseline.}
  \label{tab:norm_lr_ppl}
\end{table}

\begin{figure}[H]
    \centering
    \includegraphics[width=0.7\linewidth]{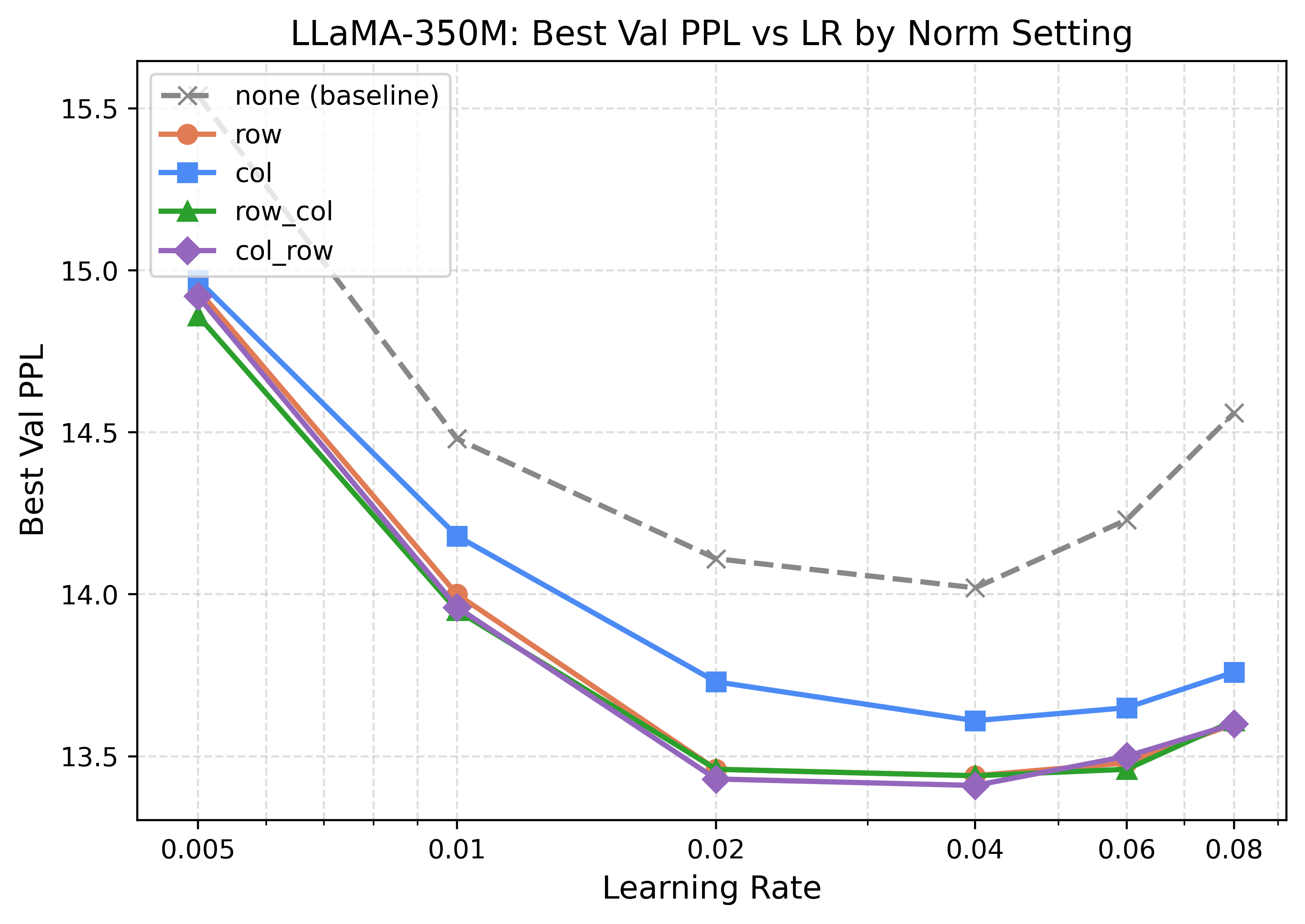}
    \caption{Visualization of sweep result for LLaMA-350M}
    \label{fig:sweep_llama-350M}
\end{figure}

\begin{table}[h]
\centering
\begin{tabular}{lcc}
\toprule
Model & Muon & MUON+ \\
\midrule
GPT-Base   & 21.68 $\pm$ 0.050 & 19.98 $\pm$ 0.031 \\
LLaMA-350M & 14.06 $\pm$ 0.046 & 13.42 $\pm$ 0.020\\
\bottomrule
\end{tabular}
\caption{Mean validation perplexity over 5 random seeds. We report mean $\pm$ standard deviation for GPT-Base and LLaMA-350M. Lower is better. Hyperparameters are consistent with Table~\ref{tab:norm_lr_ppl} and ~\ref{tab:llama_all_norm_lr}.}
\label{tab:var_bar}
\end{table}
\paragraph{Remark.}
The proposed normalization changes the Frobenius norm of the update matrix and may thus induce an implicit layer-wise learning-rate rescaling. One may worry that the gains of \textsc{Muon+} come from this effect rather than from reduced parameter-space variance. However, our learning-rate sweeps show that the optimal learning rate of \textsc{Muon+} is essentially unchanged from that of Muon (Tables~\ref{tab:llama_all_norm_lr} and~\ref{tab:norm_lr_ppl}, Figure~\ref{fig:sweep_llama-350M}), suggesting that the improvement is not due to a shift in the effective learning rate. We further verify this by matching the Frobenius norm of the normalized update to its pre-normalization value. The results are shown in Table~\ref{tab:match_F_norm}.

\begin{table}[h]
\centering
\begin{tabular}{lcc}
\toprule
Model & Muon & MUON+ \\
\midrule
GPT-Base   & 21.70  & 20.16 \\
LLaMA-350M & 14.02 & 13.55\\
\bottomrule
\end{tabular}
\caption{Matching Frobenius norm. Hyperparameters are consistent with Table~\ref{tab:norm_lr_ppl} and ~\ref{tab:llama_all_norm_lr}.}
\label{tab:match_F_norm}
\end{table}
\clearpage
\newpage
\section{Supplementary results}
\subsection{Per-step runtime}
The results are measured on an H100 GPU with mixed precision (bf16), batch size 4, and sequence length 4096, and are reported in ms/step. Compared to Muon, all normalization variants introduce only negligible runtime overhead.
\begin{table}[h]
\centering
\begin{tabular}{lcc}
\toprule
\textbf{Method} & \textbf{Time / step} & \textbf{Ratio} \\
\midrule
Muon & 892.8 & 1.000 \\
\midrule
\textsc{Muon+} (col)      & 896.6 & 1.004 \\
\textsc{Muon+} (row)      & 900.6 & 1.008 \\
\textsc{Muon+} (col\_row) & 909.9 & 1.019 \\
\textsc{Muon+} (row\_col) & 910.3 & 1.019 \\
\bottomrule
\end{tabular}
\caption{Per-step runtime comparison between Muon and MUON+.}
\label{tab:per_step_cost}
\end{table}
\subsection{Imbalance trend}
\begin{figure}[H]
    \centering
    \includegraphics[width=\linewidth]{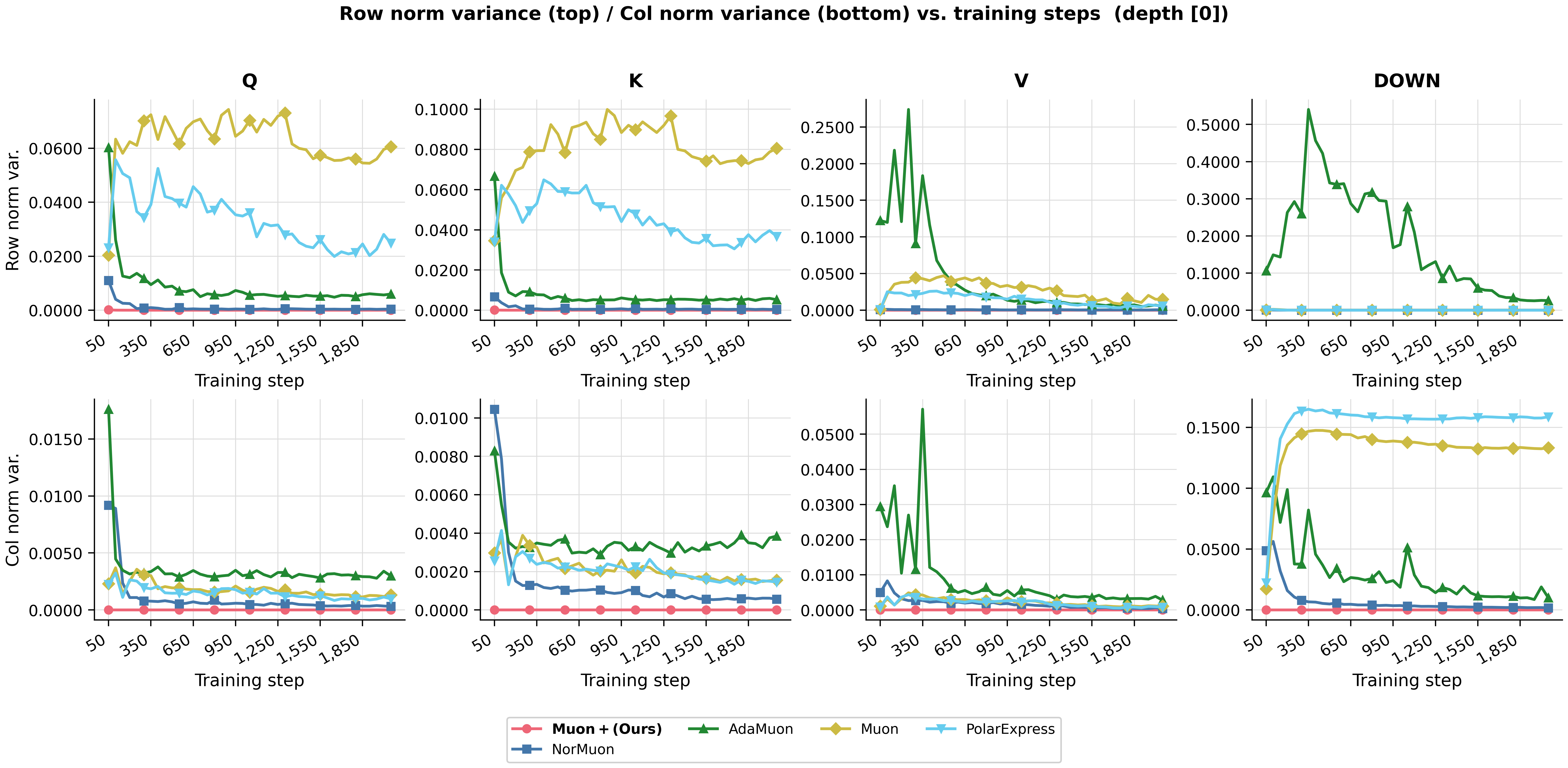}
    \includegraphics[width=\linewidth]{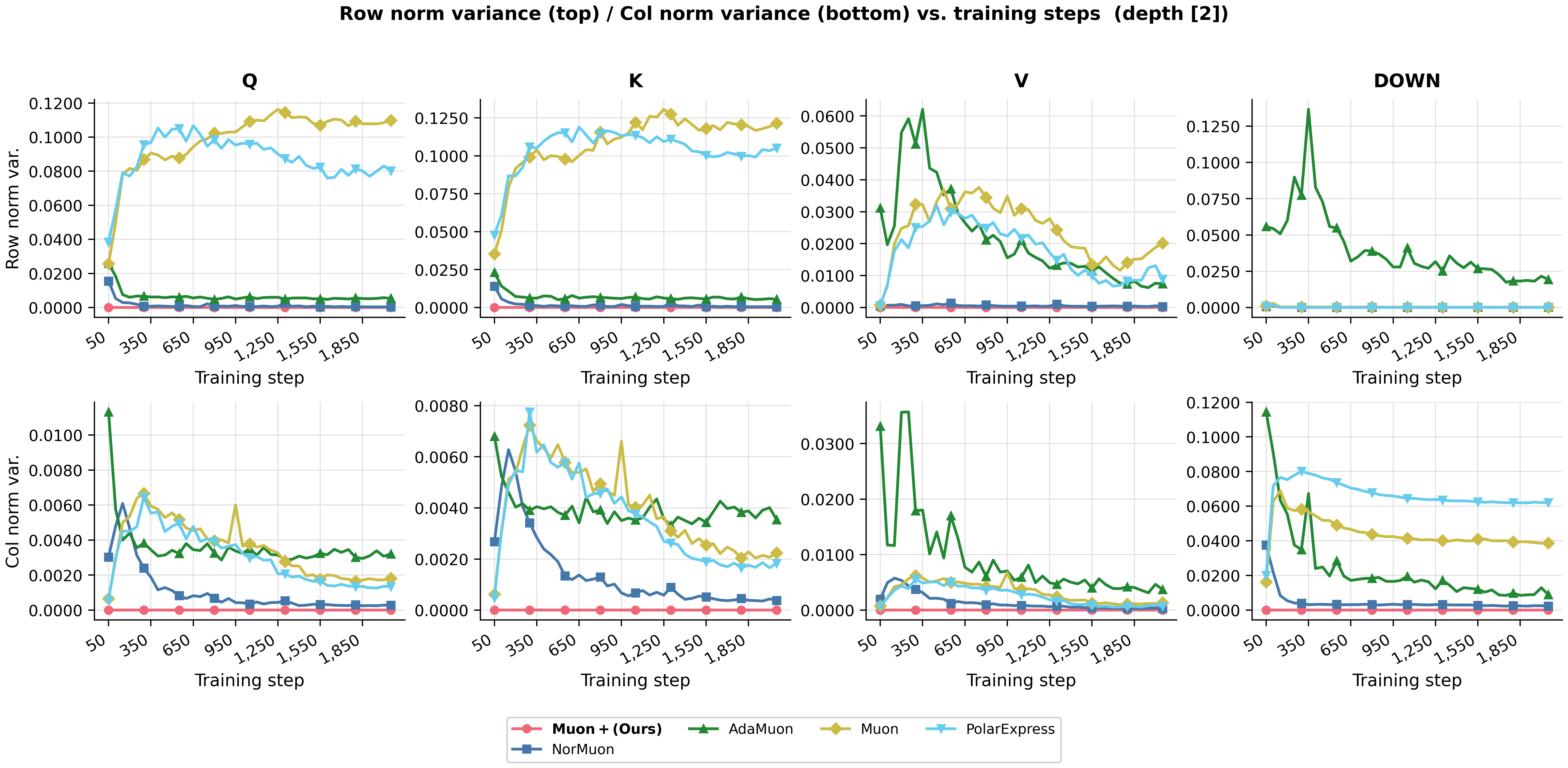}
    \includegraphics[width=\linewidth]{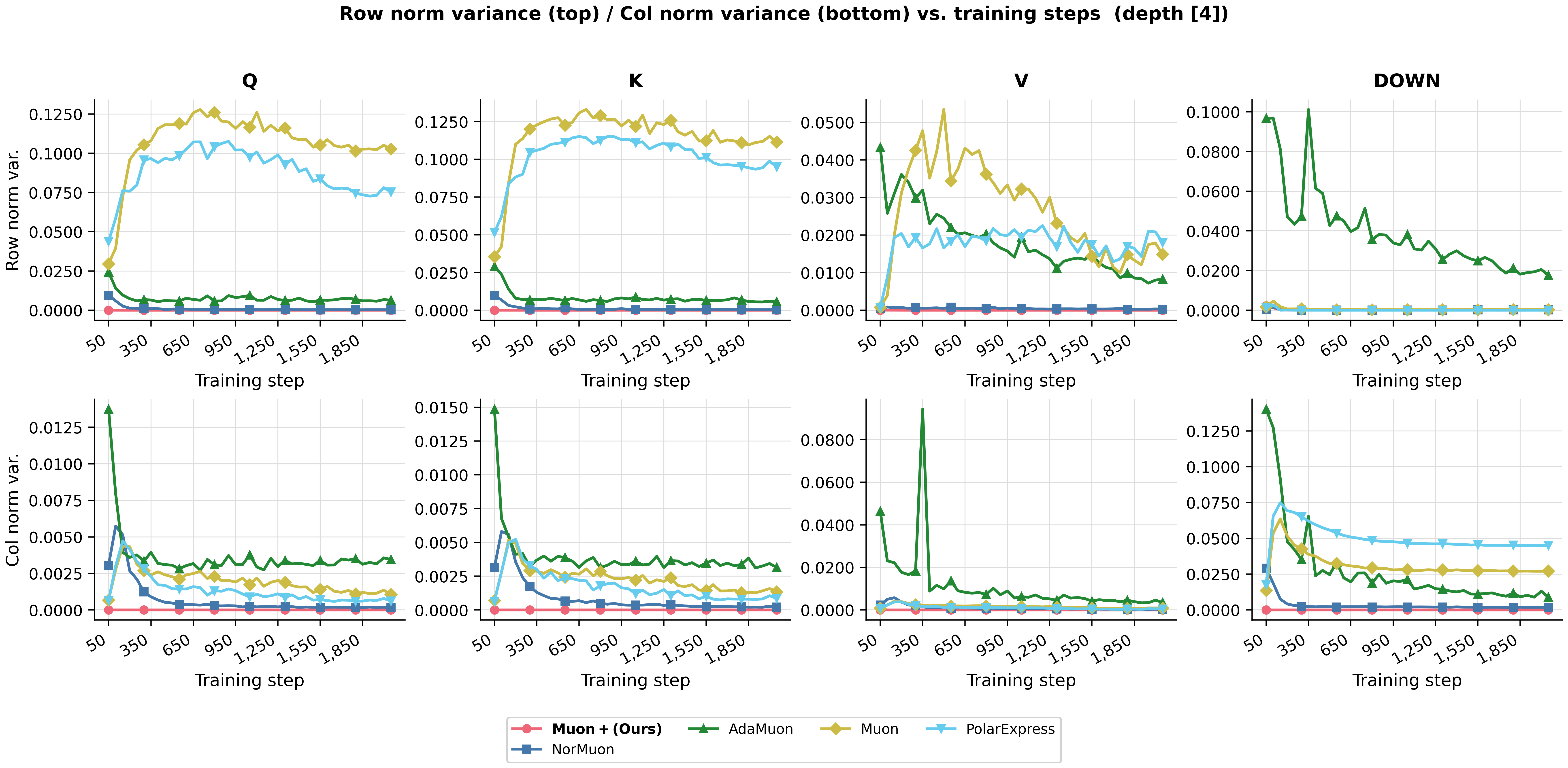}
    \caption{Detailed Imbalance Trend}
    \label{fig:var_trend_all}
\end{figure}
\end{document}